\documentclass[twoside,11pt]{article}

\usepackage{float}

\usepackage{jmlr2e}
\usepackage{lastpage}      
\hypersetup{hidelinks}    

\usepackage{amsmath}
\usepackage{float}
\usepackage{booktabs}
\usepackage[dvipsnames]{xcolor}
\definecolor{darkgreen}{rgb}{0.0,0.5,0.0}
\usepackage{graphicx}
\usepackage{multirow}
\usepackage{aliascnt}
\usepackage{cleveref}        

\newtheorem{assumption}[theorem]{Assumption}
 
\newcommand{\E}{\mathbb{E}}
\newcommand{\R}{\mathbb{R}}
\newcommand{\pold}{\pi_{\theta_{\mathrm{old}}}}
\newcommand{\pref}{\pi_{\mathrm{ref}}}
\newcommand{\KL}{D_{\mathrm{KL}}}
\newcommand{\Meff}{M_{\mathrm{eff}}}
\newcommand{\Ceff}{C_{\mathrm{eff}}}
\newcommand{\Keff}{K_{\mathrm{eff}}}
\DeclareMathOperator{\std}{std}

\usepackage{placeins}

\ShortHeadings{Predictable GRPO}{Ghosh et al.}
\firstpageno{1}

\begin{document}

\title{Predictable GRPO: A Closed-Form Model of Training Dynamics}

\author{\name Rajat Ghosh \email rajat.ghosh11@gmail.com \\
       \addr Nutanix\\
       San Jose, USA
       \AND
       \name Datta Nimmaturi\thanks{Datta Nimmaturi contributed to the preliminary ideation and early validation while working at Nutanix.} \email venkatadattasainimmaturi@gmail.com \\
       \addr Unsloth\\
       Hyderabad, India
       \AND
       \name Aryan Singhal \email aryan.singhal2001@gmail.com \\
       \addr Nutanix\\
       San Jose, USA
       \AND
       \name Vaishnavi Bhargava \email vaishnavi.bhargava2605@gmail.com \\
       \addr Nutanix\\
       San Jose, USA
       \AND
       \name Henry Wong \email henrylwongx@gmail.com \\
       \addr Nutanix\\
       San Jose, USA
       \AND
       \name Johnu George \email johnugeorge109@gmail.com \\
       \addr Nutanix\\
       Bengaluru, India
       \AND
       \name Debojyoti Dutta \email ddutta@gmail.com \\
       \addr Nutanix\\
       San Jose, USA}

\editor{} 

\maketitle

\begin{abstract}
Group Relative Policy Optimization (GRPO) has become a standard tool for
improving the reasoning ability of large language models, yet its training
dynamics are still described empirically: reward trajectories are fit with
low-parameter functional forms whose constants carry no mechanistic meaning, and
hyperparameter choices remain a matter of trial and error. We develop a
first-principles reduced-order model of these dynamics. Under a single mean-field
assumption that summarizes the policy by its expected reward, we reduce the GRPO
update to a stochastically-forced damped oscillator whose mass, damping, and
stiffness are fixed in closed form by the optimizer hyperparameters together with
a single measured curvature scale---momentum supplies the
inertia, off-policy lag erodes the damping, and the group size enters, to leading
order, as a noise temperature. The reduction has three consequences. First, it subsumes the
empirical single-exponential saturation law as its overdamped limit, recasting
the fitted plateau, timescale, and size exponent as the fixed point, inverse
stiffness, and curvature-scaling exponent of the underlying potential, and adding,
through the retained inertial term, the slow-start phase the single exponential
cannot represent. Second, it yields predictions tied to independently measurable
quantities rather than fitted ones: group-size invariance of the deterministic
trajectory with a $1/G$ stationary fluctuation, a sharp stability threshold in the
refresh interval, and an overdamped-to-oscillatory transition. Third, it furnishes
diagnostics that separate failure modes a reward curve alone conflates---reward
hacking, advantage degeneracy, policy concentration, and dynamical instability.
Across three models and two group sizes, the closed-form trajectory fits training
reward to $R^2 \geq 0.91$ and the mean trajectory is group-size invariant to
leading order---on both the reward curve and out-of-distribution transfer to eight
math benchmarks---while the within-group reward spread retains a residual
$G$-dependence that the leading-order temperature picture does not capture. The
stability and oscillatory predictions are exercised in a controlled exact-reduction
setting where the mean-field assumption holds exactly: a softmax-bandit reduction
reproduces the predicted overdamped-to-oscillatory transition and locates the
refresh-interval stability threshold at the independently measured stiffness, with
a deep-network demonstration left to future work.
\end{abstract}


\section{Introduction}

\label{sec:intro}
Large language models (LLMs) have demonstrated remarkable reasoning capabilities after RL post-training \citep{shao2024deepseekmathpushinglimitsmathematical, guo2025deepseek}, yet the mechanisms by which they acquire these capabilities remain largely opaque \citep{yue2025doesreinforcementlearningreally, wang2026new, yan2026spuriousrewardsparadoxmechanistically}. Scaling RL compute is a critical paradigm for advancing LLM capability frontier. For instance, Deepseek-R1-Zero used 100,000 H800 GPU hours for RL post-training---3.75\% of its pre-training compute \citep{guo2025deepseek}. Moreover, the trend in RL post-training compute is amplified in frontier LLMs---more than 10$\times$ increase from o1 to o3 \citep{openai2025o3o4mini} and a similar jump from Grok-3 to Grok-4 \citep{xai2025grok4}.

While RL compute scale massively, the general understanding of the field has remained more art than science. Recent breakthroughs in RL compute are generally driven by a handful of frontier labs in the form of novel algorithms \citep{yu2026dapo, zheng2025groupsequencepolicyoptimization} and training reports \citep{minimax2025minimaxm1scalingtesttimecompute, mistralai2025magistral}. These studies provide specific solutions but lack scientific understanding of RL compute scaling in general. This leads to perpetual trial-and-error mode in RL post-training which is both sub-optimal and cost-prohibitive. Our contributions are as follows: 

\begin{enumerate}
  \item \textbf{A reduced-order model of GRPO training.} Under a single
  mean-field assumption (Assumption~\ref{ass:mf}), we reduce the GRPO update
  \eqref{eq:update} to a stochastically-forced damped oscillator
  (Theorem~\ref{thm:main}) whose coefficients $\Meff, \Ceff, \Keff$ are given
  in closed form \eqref{eq:final}. The reduction is mechanical apart from this
  assumption and a controlled second-order truncation in the step size and refresh
  lag: momentum promotes the gradient flow to second order
  (Lemma~\ref{lem:momentum}), off-policy lag supplies inertia while eroding
  damping (Lemma~\ref{lem:lag}), and the group size enters, to leading order, as a
  noise temperature (Lemma~\ref{lem:noise}).

  \item \textbf{A mechanistic account of the empirical saturation law.} We show
  that the single-exponential reward law \eqref{eq:emp} fitted in prior work is
  the overdamped limit of our model (Corollary~\ref{cor:sat}), which
  reinterprets its plateau, timescale, and size exponent as the fixed point,
  inverse stiffness, and curvature-scaling exponent of the dynamics rather than
  as fitted constants \eqref{eq:curv}. Retaining the inertial term yields the
  slow-start inflection that the single-exponential form cannot represent
  (Proposition~\ref{prop:threephase}).

  \item \textbf{Falsifiable predictions pinned to measurable quantities.}
  Because the coefficients are tied to independently observable slopes, the
  model yields predictions rather than fits: leading-order group-size invariance
  of the deterministic trajectory with a $1/G$ stationary fluctuation
  (Proposition~\ref{prop:group}), a sharp stability threshold
  $K\eta K_{\mathrm{ref}}>1-\mu$ (Proposition~\ref{prop:stab}), and an
  overdamped-to-oscillatory transition under refresh-interval sweeps
  (Section~\ref{sec:saturation}).

  \item \textbf{Diagnostics that separate distinct failure modes.} We
  operationalize the model as four read-outs that a reward curve alone cannot
  disentangle---train--eval decoupling, advantage degeneracy, policy
  concentration, and dynamical instability---thereby distinguishing erosion of
  the mean-field premise (Assumption~\ref{ass:mf}) from loss of stability of the
  dynamics (Theorem~\ref{thm:main}).

  \item \textbf{Empirical validation.} Across three models and two group sizes,
  the three-phase law \eqref{eq:threephase} fits the training-reward
  trajectories with $R^2\geq0.91$, and the group-size invariance of the mean
  trajectory predicted by Proposition~\ref{prop:group} is borne out to leading
  order---on both the reward curve and out-of-distribution transfer to eight
  benchmarks---with the within-group spread retaining a residual $G$-dependence
  beyond the temperature picture. The stability and
  oscillatory predictions, untestable on the overdamped GSM8K runs, are instead
  exercised in a controlled exact-reduction setting (Section~\ref{sec:oscillatory}):
  a softmax-bandit reduction in which Assumption~\ref{ass:mf} holds exactly
  reproduces the overdamped-to-oscillatory transition and pins the refresh-interval
  threshold of Proposition~\ref{prop:stab} at the independently measured stiffness,
  leaving the deep-network demonstration to future work.
\end{enumerate}

\section{Related Work}

\label{sec:related}

\paragraph{Theory of GRPO and RLVR.} A growing line of work seeks to explain
\emph{why} reinforcement learning with verifiable rewards improves reasoning,
rather than merely demonstrating that it does. \citet{mroueh2025} characterizes
the effective loss of GRPO and establishes the ``success amplification''
property---the existence of a stable fixed point whose expected reward exceeds
that of the reference policy---together with a one-step contraction governing the
approach to it. \citet{vojnovic2025} analyze the alignment objective induced by
GRPO and identify a $\beta\sigma(\pi_\theta)$ effective-regularization form for
the KL-anchored term. We take both results as inputs rather than re-deriving
them: Assumption~\ref{ass:mf} imports the fixed point $p^\star$ and the
contraction $h'(p^\star)$ from \citet{mroueh2025}, and Lemma~\ref{lem:backbone}
inherits the stiffness scale from the effective-regularization form of
\citet{vojnovic2025}. A separate strand asks whether RLVR confers new
capabilities or sharpens existing ones \citep{yue2025doesreinforcementlearningreally,
wang2026new, yan2026spuriousrewardsparadoxmechanistically}; this question is
orthogonal to ours, which concerns the \emph{trajectory} along which a fixed
reward signal is optimized rather than the source of the reward itself.

\paragraph{Empirical scaling and saturation laws.} A complementary body of work
fits GRPO reward trajectories with low-parameter functional forms.
\citet{nimmaturi2025} propose predictive scaling laws for GRPO training, and
\citet{ghosh2026} report that across model families and scales the reward follows
the saturating law of \eqref{eq:emp}, with a size-dependent timescale $M^{0.3}$.
These laws are accurate and useful for compute planning, but their parameters are
fitted constants without mechanistic referents: the plateau, the timescale, and
the size exponent are read off the data rather than predicted. Our reduction
subsumes this line of work rather than competing with it. Corollary~\ref{cor:sat}
shows that the single-exponential law \eqref{eq:emp} is precisely the overdamped
limit of Theorem~\ref{thm:main}, and Equation~\eqref{eq:curv} reinterprets the
$M^{0.3}$ exponent as a curvature-scaling exponent of the underlying potential.
Where the inertial term is retained, Proposition~\ref{prop:threephase} further
predicts the slow-start phase that the single exponential cannot represent,
recovering the full three-phase shape these fits leave unexplained.

\paragraph{Algorithms and training reports.} Much of the recent progress in RL
post-training arrives as novel algorithms
\citep{yu2026dapo, zheng2025groupsequencepolicyoptimization} and large-scale
training reports \citep{minimax2025minimaxm1scalingtesttimecompute,
mistralai2025magistral}. These contributions advance the practical frontier and
report specific, hard-won configurations, but they are not designed to yield a
general account of how a training curve depends on its hyperparameters; the
prescriptions they offer are calibrated to a particular setting and do not, on
their own, predict behavior outside it. Our aim is complementary: rather than a
new optimizer or a new run, we provide a closed-form relationship between the
training configuration $(\beta,\eta,\mu,G,K_{\mathrm{ref}})$ and the resulting
trajectory, with the group-size invariance of Proposition~\ref{prop:group} and
the stability threshold of Proposition~\ref{prop:stab} as falsifiable
consequences.

\paragraph{Positioning.} The boundary of our contribution is therefore sharp. We
import the fixed point and contraction from \citet{mroueh2025} and the
regularization scale from \citet{vojnovic2025}; what is new is the promotion of
the resulting first-order relaxation to a second-order system, with the inertia
and damping supplied mechanically by momentum (Lemma~\ref{lem:momentum}) and
off-policy lag (Lemma~\ref{lem:lag}), and the consequent subsumption of the
empirical saturation laws \citep{nimmaturi2025, ghosh2026} as the overdamped
limit of a single equation. The reduction is conditional on the mean-field
assumption (Remark~\ref{rem:scope}) and is not an unconditional account of GRPO;
it is, rather, the minimal dynamical model that turns fitted reward curves into
derived ones.

\section{Formalism}
\label{sec:formalism}
\subsection{Setup}

For a prompt $q$, GRPO samples a group of $G$ completions
$\{o_i\}_{i=1}^{G}\sim\pold(\cdot\mid q)$, scores them with a reward
$r_i=R(q,o_i)$, and forms the group-relative advantage
\begin{equation}
A_i \;=\; \frac{r_i-\bar r}{\std(r)},
\qquad \bar r=\tfrac1G\textstyle\sum_j r_j .
\label{eq:adv}
\end{equation}
The (unclipped) objective, with KL anchoring weight $\beta>0$ to a reference policy
$\pref$, is
\begin{equation}
J(\theta)
=\E_{q}\,\E_{\{o_i\}\sim\pold}\!\left[\frac1G\sum_{i=1}^{G}
\frac{\pi_\theta(o_i\mid q)}{\pold(o_i\mid q)}\,A_i\right]
-\beta\,\E_{q}\!\left[\KL\!\big(\pi_\theta\,\|\,\pref\big)\right].
\label{eq:obj}
\end{equation}

Parameters are updated by stochastic ascent with learning rate $\eta$ and momentum
coefficient $\mu\in[0,1)$ (the role played by Adam's $\beta_1$), and the sampling
policy is refreshed every $K_{\mathrm{ref}}$ steps:
\begin{equation}
\theta_{k+1}=\theta_k+\eta\,\widehat{\nabla J}(\theta_k;\theta_{\mathrm{old}})
+\mu\,(\theta_k-\theta_{k-1}),
\qquad \theta_{\mathrm{old}}=\theta_{k-K_{\mathrm{ref}}} .
\label{eq:update}
\end{equation}
We use $\delta$ for the step-to-time conversion. The discrete map \eqref{eq:update}
is the object of interest; the continuous-time equations below are its \emph{modified}
(equivalent) equation in the sense of backward-error analysis, which we make precise
as follows. The reductions are taken after linearizing the drift at the fixed point
$p^\star$ (Lemma~\ref{lem:backbone}); throughout we assume $g\in C^3$ in a
neighborhood of $p^\star$, so that $K=-c\,g'(p^\star)$ is well defined and the
neglected Taylor remainders are bounded. For a fixed horizon $T>0$, we say an ODE
\emph{tracks} the map to order $O(\delta^2)$ if its solution $e(t)$, initialized to
match $(e_0,e_1)$, satisfies
$\max_{0\le k\le T/\delta}\lvert e_k-e(k\delta)\rvert=O(\delta^2)$ as $\delta\to0$.
Setting $\delta=1$ in the final expressions fixes the time unit as one optimizer step,
so the effective coefficients are measured in step-units and the genuine small
parameters controlling the truncation are the per-step contraction $\eta K$ and the
refresh lag $\tau$ rather than $\delta$ itself.
 
\subsection{The mean-field reduction}
 
\begin{assumption}[Scalar order parameter]
\label{ass:mf}
Let $p(\theta)=\E_{o\sim\pi_\theta}[R]$ denote the expected reward. There exists a
$C^2$ drift $g:[0,1]\to\R$ such that the gradient-flow projection of
\eqref{eq:update} satisfies $\dot p=g(p)+o(1)$ as $\delta\to0$, and the dynamics of
$\theta$ are summarized by $p$ to leading order. Moreover $g$ admits a unique stable
fixed point $p^\star\in(p_{\mathrm{ref}},1)$, with one-step contraction
$h'(p^\star)=1+\delta\,g'(p^\star)\in(-1,1)$.
\end{assumption}
 
\noindent
Assumption~\ref{ass:mf} is the only non-mechanical step in what follows. It holds
exactly for tabular/softmax parameterizations and is an approximation for deep
networks. The existence of $p^\star$ in the interior above $p_{\mathrm{ref}}$ (the
``success amplification'' property) and the value of the contraction $h'(p^\star)$
are taken from \citet{mroueh2025}; we do not re-derive them.
 
\subsection{Derivation}
 
\begin{lemma}[Policy-gradient flow]
\label{lem:pg}
At the on-policy evaluation point $\theta=\theta_{\mathrm{old}}$, the importance
ratios in \eqref{eq:obj} equal one and
\begin{equation}
\nabla_\theta J
=\E_{q}\E_{\{o_i\}}\!\left[\frac1G\sum_i A_i\,\nabla_\theta\log\pi_\theta(o_i\mid q)\right]
-\beta\,\nabla_\theta\,\E_q \KL\!\big(\pi_\theta\,\|\,\pref\big).
\label{eq:pg}
\end{equation}
\end{lemma}
\begin{proof}
Differentiate \eqref{eq:obj} and apply the score-function identity
$\nabla_\theta(\pi_\theta/\pold)=(\pi_\theta/\pold)\nabla_\theta\log\pi_\theta$;
evaluating at $\theta=\theta_{\mathrm{old}}$ sets $\pi_\theta/\pold=1$.
\end{proof}
 
\begin{lemma}[First-order backbone]
\label{lem:backbone}
Under Assumption~\ref{ass:mf}, writing $e=p-p^\star$ and linearizing $g$ at
$p^\star$ gives
\begin{equation}
c\,\dot e=-K\,e,
\qquad K\;\propto\;\beta\,\bigl(1-h'(p^\star)\bigr),
\label{eq:first}
\end{equation}
with relaxation rate $\lambda=K/c=-g'(p^\star)$.
\end{lemma}
\begin{proof}
$g(p)=g'(p^\star)e+O(e^2)$ and $g(p^\star)=0$, so $\dot e=g'(p^\star)e$. Stability
gives $g'(p^\star)<0$; set $\lambda=-g'(p^\star)>0$ and $c$ the discretization
constant so that $K=\lambda c$. The restoring coefficient $K$ is the curvature of the
KL-regularized effective potential at $p^\star$: the reward drift fixes the
contraction $1-h'(p^\star)$ through $g'(p^\star)$, while the $\KL$ term in
\eqref{eq:pg} sets the local stiffness scale with multiplicative weight $\beta$
(consistent with the $\beta\sigma(\pi_\theta)$ effective-regularization form of
\citet{vojnovic2025}). This gives the stated proportionality.
\end{proof}
 
\begin{lemma}[Momentum promotes the flow to second order]
\label{lem:momentum}
The linear heavy-ball map $e_{k+1}-e_k=-\eta K e_k+\mu(e_k-e_{k-1})$ is tracked to
order $O(\delta^2)$, in the sense above, by the second-order ODE
$m_\mu\ddot e+c_\mu\dot e+Ke=0$ with
\begin{equation}
m_\mu=\frac{(1+\mu)\,\delta^2}{2\eta},
\qquad
c_\mu=\frac{(1-\mu)\,\delta}{\eta}.
\label{eq:hb}
\end{equation}
\end{lemma}
\begin{proof}
The map is a constant-coefficient linear recurrence with characteristic equation
$z^2-(1+\mu-\eta K)z+\mu=0$. Its two modes are matched by the continuous modes
$z=e^{s\delta}$ of $m_\mu s^2+c_\mu s+K=0$: expanding
$e^{\pm s\delta}=1\pm s\delta+\tfrac12 s^2\delta^2+O(\delta^3)$ reproduces this ODE at
$O(\delta^2)$. Equivalently, inserting
$e_{k\pm1}=e\pm\delta\dot e+\tfrac{\delta^2}{2}\ddot e+O(\delta^3)$ into the map, the
$O(\delta^2)$ terms give $\tfrac{\delta^2}{2}(1+\mu)\ddot e$, the $O(\delta)$ terms
$\delta(1-\mu)\dot e$, and the force term $\eta K e$; dividing by $\eta$ yields
\eqref{eq:hb}. Because the recurrence is linear, its two discrete modes and the two
continuous modes agree to $O(\delta^3)$ per step, so the error accumulated over
$T/\delta$ steps is $O(\delta^2)$ uniformly on $[0,T]$, which is tracking in the
stated sense. (Verified symbolically; see Appendix~\ref{app:symbolic}.)
\end{proof}
 
\begin{lemma}[Off-policy lag adds inertia and removes damping]
\label{lem:lag}
With the restoring force evaluated at the refreshed policy, i.e.\ at the delayed
state $e(t-\tau)$ with $\tau=K_{\mathrm{ref}}\,\delta$, the delay equation
$m_\mu\ddot e+c_\mu\dot e+K\,e(t-\tau)=0$ is approximated, to order $O(\tau^3)$, by
$\Meff\ddot e+\Ceff\dot e+\Keff e=0$ with
\begin{equation}
\Meff=m_\mu+\tfrac12 K\tau^2,
\qquad
\Ceff=c_\mu-K\tau,
\qquad
\Keff=K,
\label{eq:effcoef}
\end{equation}
valid when the lag is small against the relaxation time, $\lambda\tau=(K/c)\tau\ll1$.
\end{lemma}
\begin{proof}
Substitute
$e(t-\tau)=e-\tau\dot e+\tfrac{\tau^2}{2}\ddot e-\tfrac{\tau^3}{6}\dddot e+O(\tau^4)$
into $m_\mu\ddot e+c_\mu\dot e+K\,e(t-\tau)=0$ and collect by derivative order: the
$\ddot e$ coefficient gains $\tfrac12K\tau^2$, the $\dot e$ coefficient loses $K\tau$,
and the $e$ coefficient is unchanged. The leading neglected term is
$-\tfrac16 K\tau^3\dddot e$, bounded on the linearized trajectory by the mode
frequencies, so \eqref{eq:effcoef} holds up to $O(\tau^3)$ in the same uniform sense
as Lemma~\ref{lem:momentum}. The expansion replaces the delay equation by a
finite-order ODE, and retaining the $\ddot e$ term is what promotes the system to
second order; the resulting $\zeta=1$ onset is consequently conservative, an effect
quantified directly against the exact dynamics in Section~\ref{sec:oscillatory}.
(Verified symbolically; see Appendix~\ref{app:symbolic}.)
\end{proof}
 
\begin{lemma}[Group size is a temperature]
\label{lem:noise}
Let $s_i=A_i\,\nabla_\theta\log\pi_\theta(o_i\mid q)$ be the per-completion
advantage-weighted score, with population mean $\nabla J$ and covariance $\Sigma$ that
scales, to leading order, with the within-group advantage spread,
$\Sigma\propto\sigma_{\mathrm{adv}}^2$. Then the $G$-sample estimator obeys
$\widehat{\nabla J}=\nabla J+G^{-1/2}\zeta$ with $\zeta\Rightarrow\mathcal N(0,\Sigma)$
as $G\to\infty$, and its projection onto the order parameter is a scalar forcing of
variance $D_0/G$ with $D_0\propto\sigma_{\mathrm{adv}}^2$. In the diffusion
approximation this forcing is the white noise $\sqrt{D/G}\,\xi(t)$ of
Theorem~\ref{thm:main}, with $D\propto D_0$.
\end{lemma}
\begin{proof}
The advantages \eqref{eq:adv} are formed from $G$ i.i.d.\ completions, so the
Lindeberg central limit theorem gives $\widehat{\nabla J}-\nabla J=G^{-1/2}\zeta$ with
$\zeta$ asymptotically $\mathcal N(0,\Sigma)$. Writing
$u=\nabla_\theta p/\lVert\nabla_\theta p\rVert$ for the unit direction along which
$e=p-p^\star$ responds to $\theta$, the scalar forcing $u^\top\zeta$ has variance
$D_0=u^\top\Sigma u\propto\sigma_{\mathrm{adv}}^2$, and the projection preserves the
$1/G$ scaling. Passing from this martingale-difference forcing to a
$\delta$-correlated white noise is a diffusion (weak-convergence) approximation---the
standard SGD-as-SDE closure, exact in the small-step limit under the usual
martingale-CLT conditions; it is the one genuinely non-algebraic step in the
reduction, and we do not re-derive it here.
\end{proof}
 
\begin{theorem}[Reduced-order GRPO dynamics]
\label{thm:main}
Under Assumption~\ref{ass:mf}, the order parameter $e=p-p^\star$ evolves, to second
order in $\delta$ and $\tau$, as the stochastically forced damped oscillator
\begin{equation}
\boxed{\;\Meff\,\ddot e+\Ceff\,\dot e+\Keff\,e=\sqrt{D/G}\;\xi(t)\;}
\label{eq:main}
\end{equation}
with $\langle\xi(t)\xi(t')\rangle=\delta(t-t')$ and the coefficients of
\eqref{eq:effcoef}. At $\delta=1$,
\begin{equation}
\Meff=\frac{(1+\mu)+K\eta\tau^2}{2\eta},
\qquad
\Ceff=\frac{(1-\mu)-K\eta\tau}{\eta},
\qquad
\Keff=K .
\label{eq:final}
\end{equation}
\end{theorem}
\begin{proof}
Combine Lemmas \ref{lem:pg}--\ref{lem:noise}: Lemma~\ref{lem:backbone} supplies the
restoring term, Lemma~\ref{lem:momentum} the inertia and base damping,
Lemma~\ref{lem:lag} the delay corrections, and Lemma~\ref{lem:noise} the forcing.
Substituting \eqref{eq:hb} into \eqref{eq:effcoef} and setting $\delta=1$ gives
\eqref{eq:final}.
\end{proof}
 
\subsection{Consequences}

\begin{proposition}[Damping ratio and the overdamped regime]
\label{prop:zeta}
The damping ratio of \eqref{eq:main} is
\begin{equation}
\zeta=\frac{\Ceff}{2\sqrt{\Keff\Meff}}
=\frac{(1-\mu)-K\eta\tau}{\sqrt{2K\eta\bigl[(1+\mu)+K\eta\tau^2\bigr]}} .
\label{eq:zeta}
\end{equation}
In the absence of momentum and lag $(\mu=0,\tau=0)$, $\zeta=(2K\eta)^{-1/2}$, which
exceeds $1$ for the usual small $\eta$. The system is then overdamped: the inertial
mode decays immediately and $e$ relaxes on the slow pole
$-\Keff/\Ceff=-K\eta$, recovering the first-order law of
Lemma~\ref{lem:backbone}.
\end{proposition}
\begin{proof}
Direct substitution of \eqref{eq:final}. For $\mu=\tau=0$,
$\Ceff=1/\eta$, $\Meff=1/(2\eta)$, $\Keff=K$, whence
$\zeta=(1/\eta)/\!\bigl(2\sqrt{K/(2\eta)}\bigr)=(2K\eta)^{-1/2}$. When
$\zeta\gg1$ the roots separate, with product $\Keff/\Meff=2K\eta$, into a slow pole
$-\Keff/\Ceff=-K\eta$ and a fast pole $-\Ceff/\Meff=-2$; the fast mode decays
immediately and the slow pole governs the observed relaxation.
\end{proof}
 
\begin{proposition}[Stability threshold in the refresh interval]
\label{prop:stab}
The effective damping is negative, and \eqref{eq:main} is linearly unstable, iff
\begin{equation}
K\,\eta\,K_{\mathrm{ref}}>1-\mu .
\label{eq:threshold}
\end{equation}
\end{proposition}
\begin{proof}
With $\delta=1$, $\tau=K_{\mathrm{ref}}$, and $\Ceff=\bigl[(1-\mu)-K\eta\tau\bigr]/\eta$
from \eqref{eq:final}, the sign condition $\Ceff<0$ is equivalent to
$K\eta K_{\mathrm{ref}}>1-\mu$. A negative damping coefficient places the
characteristic roots in the right half-plane.
\end{proof}
 
\begin{proposition}[Group-size invariance of the deterministic dynamics]
\label{prop:group}
The fixed point $p^\star$ and the deterministic coefficients \eqref{eq:final} are
independent of $G$; $G$ enters \eqref{eq:main} only through the forcing. Consequently,
in the stable regime $K\eta K_{\mathrm{ref}}<1-\mu$ of
Proposition~\ref{prop:stab}, the stationary fluctuation of the order parameter obeys
\begin{equation}
\langle e^2\rangle_\infty\;\propto\;\frac{D}{G\,\Ceff\,\Keff}\;\propto\;\frac1G ,
\label{eq:floor}
\end{equation}
while the mean trajectory is unchanged.
\end{proposition}
\begin{proof}
By Lemma~\ref{lem:noise}, $G$ enters \eqref{eq:main} only through the forcing
intensity $D/G$; the coefficients \eqref{eq:final} are $G$-free, so the fixed point
and the deterministic trajectory are independent of $G$. For the fluctuation,
linearize about $p^\star$ and write \eqref{eq:main} in phase space $x=(e,\dot e)$ as
$\dot x=\mathbf A x+\mathbf b\,\xi$ with
$\mathbf A=\left(\begin{smallmatrix}0&1\\-\Keff/\Meff&-\Ceff/\Meff\end{smallmatrix}\right)$
and $\mathbf b=(0,\sqrt{D/G}/\Meff)^\top$. In the stable regime $\mathbf A$ is Hurwitz
and the stationary covariance solves the Lyapunov equation
$\mathbf A\Sigma_\infty+\Sigma_\infty\mathbf A^\top+\mathbf b\mathbf b^\top=0$, whose
position entry is the standard damped-oscillator result
$\langle e^2\rangle_\infty=D/(2G\,\Ceff\,\Keff)$, in which the effective mass $\Meff$
cancels. Hence $\langle e^2\rangle_\infty\propto1/G$ while the mean trajectory is
unchanged. This is a leading-order statement about the linearized dynamics: it treats
$\sigma_{\mathrm{adv}}$, hence $D$, as an exogenous $G$-independent constant, an
approximation whose empirical limits are examined in Remark~\ref{rem:corr} and
Section~\ref{sec:results}.
\end{proof}
 
\begin{remark}[Empirical correspondences]
\label{rem:corr}
Proposition~\ref{prop:zeta} predicts that runs with small learning rate, no
momentum, and frequent refresh are overdamped and exhibit monotone saturation to
$p^\star$, matching the regime in which a first-order fixed-point fit attains high
$R^2$. Proposition~\ref{prop:stab} predicts divergence past a sharp refresh-interval
threshold, consistent with the instability of stale-policy training.
Proposition~\ref{prop:group} predicts that reducing $G$ leaves the mean trajectory
intact and raises only the variance, consistent with reports that two rollouts per
prompt match sixteen up to noise. This last correspondence is a leading-order
statement: it treats the within-group spread $\sigma_{\mathrm{adv}}$ as an
exogenous, $G$-independent constant, whereas empirically $\sigma_{\mathrm{adv}}$
drifts and carries a residual $G$-dependence of its own
(Section~\ref{sec:results}), so the mean invariance is the robust prediction while
the $1/G$ variance law holds only to the extent that this leading-order treatment
of $\sigma_{\mathrm{adv}}$ does.
\end{remark}
 
\begin{remark}[Scope]
\label{rem:scope}
Theorem~\ref{thm:main} rests on two modeling steps, not one. The first is
Assumption~\ref{ass:mf}, the mean-field reduction to a scalar order parameter;
the second is the backward-error (modified-equation) expansion of
Lemmas~\ref{lem:momentum} and~\ref{lem:lag}, truncated at second order in the step
size $\delta$ and the refresh lag $\tau$, which drops $O(\delta^3,\tau^3)$ terms
and is therefore accurate only for small $\delta$ and $\tau$. Within that
truncation the deterministic algebra of Lemmas~\ref{lem:pg}--\ref{lem:lag} is exact
(verified symbolically; see Appendix~\ref{app:symbolic}), the sole non-algebraic step
being the diffusion (white-noise) approximation of Lemma~\ref{lem:noise}, and
Section~\ref{sec:oscillatory} quantifies the truncation error directly against the
exact dynamics. The
coefficients are moreover local linearizations at $p^\star$, so \eqref{eq:main}
describes the approach to the fixed point rather than global dynamics, and the
curvature $K$ is identifiable from the early-training slope $\lambda=K/c$ rather
than predicted. The appropriate reading is a reduced-order model conditional on
the mean-field reduction and the small-step expansion, not an unconditional
account of GRPO.
\end{remark}
 
\subsection{The empirical reward-saturation law as an overdamped limit}
\label{sec:saturation}
 
The reduced-order equation \eqref{eq:main} makes contact with empirical work that fits
GRPO reward trajectories with simple parametric forms
\citep{nimmaturi2025,ghosh2026}. In particular, \citet{ghosh2026} report that across
model families and scales the reward follows the saturating law
\begin{equation}
R(t,M)=R_\infty\bigl(1-e^{-t/M^{0.3}}\bigr),
\label{eq:emp}
\end{equation}
with a steady-state reward $R_\infty$ that grows with model size $M$ and a single
size-dependent timescale $M^{0.3}$. We show that \eqref{eq:emp} is the overdamped
limit of Theorem~\ref{thm:main}, and that retaining the inertial term resolves the
early ``slow-start'' phase that \eqref{eq:emp} cannot represent.
 
Throughout we identify reward with the order parameter, $R\equiv p$, write
$R_\infty\equiv p^\star$ for the plateau, and take the trajectory to start from rest,
$\dot p(0)=0$, at initial reward $R_0$. By Proposition~\ref{prop:group} the forcing
sets only the stationary fluctuation, so the mean trajectory obeys the deterministic
part of \eqref{eq:main}.
 
\begin{corollary}[Saturation law]
\label{cor:sat}
In the overdamped regime $\zeta\ge1$, the mean reward of Theorem~\ref{thm:main} is
\begin{equation}
R(t)=R_\infty-(R_\infty-R_0)\,e^{-\lambda t},
\qquad \lambda=\frac{\Keff}{\Ceff},
\label{eq:sat}
\end{equation}
and, when $R_0\ll R_\infty$, reduces to the empirical form \eqref{eq:emp} with
timescale $\tau=1/\lambda=\Ceff/\Keff=\bigl[(1-\mu)-K\eta K_{\mathrm{ref}}\bigr]/(\eta K)$.
\end{corollary}
\begin{proof}
When $\zeta\gg1$ the two roots of \eqref{eq:main} separate
(Proposition~\ref{prop:zeta}); after the fast mode decays the slow pole
$-\Keff/\Ceff$ governs, so $e(t)=e(0)\,e^{-\lambda t}$ with $e=R-R_\infty$.
Substituting $e(0)=R_0-R_\infty$ gives \eqref{eq:sat}; dropping the $R_0/R_\infty$
term gives \eqref{eq:emp}.
\end{proof}
 
\noindent
Corollary~\ref{cor:sat} supplies mechanistic content for the three quantities in
\eqref{eq:emp}. The plateau is the GRPO fixed point, $R_\infty=p^\star(M)$, so its
growth with scale is the success amplification of \citet{mroueh2025} rather than a
fitted constant. The timescale is an inverse stiffness, $\tau=1/(\eta K)$ at
$\mu=K_{\mathrm{ref}}=0$, so the empirical exponent is a \emph{curvature}-scaling
exponent,
\begin{equation}
\tau\propto M^{0.3}\;\Longleftrightarrow\;K(M)\propto \eta^{-1}M^{-0.3}:
\label{eq:curv}
\end{equation}
a larger model approaches a higher fixed point along a flatter potential, consistent
with lazy/NTK-regime slowdown. Finally, the initial reward $R_0$ enters \eqref{eq:sat}
as the amplitude and is absent from \eqref{eq:emp} only through the approximation
$R_0\ll R_\infty$.
 
\begin{proposition}[Three-phase refinement]
\label{prop:threephase}
Retaining the inertial term, the critically-damped ($\zeta=1$) mean trajectory from
rest is
\begin{equation}
\frac{R(t)}{R_\infty}=1-\bigl(1+\omega_n t\bigr)e^{-\omega_n t},
\qquad \omega_n=\sqrt{\Keff/\Meff},
\label{eq:threephase}
\end{equation}
which satisfies $R'(0)=0$ and has an inflection at $t^\star=1/\omega_n$. The general
overdamped case replaces \eqref{eq:threephase} by the two-exponential response with
poles $s_{1,2}=-\omega_n\bigl(\zeta\mp\sqrt{\zeta^2-1}\bigr)$.
\end{proposition}
\begin{proof}
The step response of \eqref{eq:main} from rest at $\zeta=1$ has the repeated pole
$-\omega_n$, yielding \eqref{eq:threephase}. Differentiating,
$R'(t)=R_\infty\,\omega_n^2\,t\,e^{-\omega_n t}$, so $R'(0)=0$ and $R''(t^\star)=0$ at
$t^\star=1/\omega_n$.
\end{proof}
 
\noindent
Equation~\eqref{eq:threephase} reproduces the full three-phase trajectory reported by
\citet{ghosh2026}---a slow start ($R'(0)=0$), a rapid phase centered at the inflection
$t^\star=1/\omega_n$, and a plateau---whereas the single exponential \eqref{eq:emp} is
concave with maximal slope at $t=0$ and so has no slow-start phase. The inertia
responsible for the inflection is exactly the momentum-plus-lag mass $\Meff$ of
Lemma~\ref{lem:lag}. For empirical practice this implies that the quantity expected to
collapse across scales is not $R/R_\infty$ but the doubly normalized deviation
\begin{equation}
\frac{R(t)-R_0}{R_\infty-R_0}=1-(1+\omega_n t)\,e^{-\omega_n t},
\label{eq:collapse}
\end{equation}
which removes both the reward scale and the initial offset.
 
\subsection{A falsifiable prediction}
 
Fix $(\beta,\eta,\mu,G)$ and sweep the refresh interval $K_{\mathrm{ref}}$.
Equation~\eqref{eq:zeta} predicts a monotone-to-oscillatory transition in the
learning curve at $\zeta=1$, and divergence at the threshold \eqref{eq:threshold};
both critical values are pinned by the independently measurable slope $K$. We test
this prediction directly in Section~\ref{sec:oscillatory}, in an exact softmax
reduction where Assumption~\ref{ass:mf} holds without approximation.

\section{Experimental Setup}
\label{sec:experiments}
This section describes the experimental protocol used to test the reduced-order
model of Section~\ref{sec:formalism}. The design follows directly from what the
theory asks of the data. Theorem~\ref{thm:main} predicts a specific closed-form
trajectory toward a fixed point, so we need reward curves logged densely enough
to fit a three-parameter law; Proposition~\ref{prop:group} predicts that the
\emph{deterministic} trajectory is invariant to the group size $G$ while the
stationary fluctuation scales as $1/G$, so we run every configuration at two
group sizes; and the four diagnostics of Section~\ref{sec:intro} read out
internal training quantities---advantage spread, entropy, KL, gradient norm, and
the importance-sampling ratio---so we log each of them at every step rather than
only the scalar reward.

\subsection{Models}
\label{subsec:exp_models}

We study three open-weight models chosen to span the regimes in which
Assumption~\ref{ass:mf} is expected to hold to differing degrees:
\begin{itemize}
  \item \textbf{Nemotron-Mini-4B} \citep{nvidia2024nemotronmini}, a $4$B
  general-purpose instruction-tuned model;
  \item \textbf{DeepSeek-LLM-7B-Chat} \citep{deepseekai2024deepseekllm}, a $7$B
  chat model, the largest in our study and the one without explicit reasoning
  distillation;
  \item \textbf{DeepSeek-R1-Distill-Qwen-1.5B} \citep{guo2025deepseek}, a $1.5$B
  model distilled from R1 reasoning traces and therefore the only
  reasoning-specialized model of the three.
\end{itemize}
The spread in scale ($1.5$B--$7$B) and in reasoning specialization lets us
separate effects of the dynamics from effects of the base capability: the
reasoning-distilled model is the one for which the scalar order-parameter
reduction is expected to remain cleanest, a prediction we test directly through
the policy-concentration diagnostic of
Section~\ref{subsec:gpro_training_dynamics} onward.

\subsection{Training Setup}
\label{subsec:exp_training}

All three models are post-trained with GRPO \citep{shao2024deepseekmathpushinglimitsmathematical}
on the GSM8K training split \citep{cobbe2021gsm8k} of grade-school word problems,
using a binary verifiable reward that checks the parsed final answer against the
gold solution. Training uses parameter-efficient LoRA adapters
\citep{hu2021lora} (rank $r=16$, $\alpha=64$, dropout $0.05$) applied to the
query, key, value, output, and MLP up/down/gate projections, with the adapters
merged into the base weights for evaluation. The full configuration is given in
Table~\ref{tab:gsm8k_config} (Appendix~\ref{app:gsm8k_hyperparams}); the settings
most relevant to the theory are the learning rate $\eta=1\times10^{-5}$, the KL
anchoring weight $\beta=0.005$, the clipping coefficient $\epsilon=0.2$, the
sampling temperature $0.8$, and a training horizon of two epochs. The small
learning rate is deliberate: by Proposition~\ref{prop:zeta} it places the
product $K\eta K_{\mathrm{ref}}$ well below $1-\mu$, so every run is expected to
sit in the overdamped regime where the closed-form saturation law of
Corollary~\ref{cor:sat} applies. Probing the under-damped and unstable regimes of
Propositions~\ref{prop:stab}--\ref{prop:threephase} would require raising $\eta$
or the refresh interval $K_{\mathrm{ref}}$, which we leave to future work.

\paragraph{Group sizes and seeds.} The central manipulation is the GRPO group
size, the number of completions $G$ sampled per prompt to form the
group-relative advantage of Eq.~\eqref{eq:adv}. We run every model at
$G\in\{4,16\}$, a fourfold change in rollout budget, so that the group-size
invariance of Proposition~\ref{prop:group} can be tested as a controlled
comparison. To separate the deterministic trajectory from run-to-run noise---the
$1/G$ stationary fluctuation that the theory attributes to the forcing
term---each configuration is run over two random seeds $\{42, 3407\}$, and all
reported curves are seed-averaged with $\pm$ one standard deviation bands.

\paragraph{Logged quantities.} {\sloppy Beyond the training and held-out reward, we log at
each optimizer step the quantities required by the four diagnostics: the
within-group reward standard deviation $\sigma_{\mathrm{adv}}$
(\texttt{train/\allowbreak reward\_std}), the collapsed-prompt fraction
(\texttt{train/\allowbreak frac\_\allowbreak reward\_\allowbreak zero\_\allowbreak std}), the token-level policy entropy
(\texttt{train/entropy}), the KL to the reference
(\texttt{train/kl}), the gradient norm (\texttt{train/\allowbreak grad\_norm}), and
the maximum importance-sampling ratio
(\texttt{train/\allowbreak sampling/\allowbreak importance\_sampling\_\allowbreak ratio/\allowbreak max}). These are the
empirical witnesses for, respectively, the noise temperature of
Lemma~\ref{lem:noise}, the advantage degeneracy that breaks Eq.~\eqref{eq:adv},
the scalar-reduction premise of Assumption~\ref{ass:mf}, and the off-policy lag
and stability threshold of Lemma~\ref{lem:lag} and Proposition~\ref{prop:stab}.\par}

\subsection{Evaluation Protocol}
\label{subsec:exp_eval}

We evaluate along two axes. \emph{In-distribution}, we track the held-out GSM8K
reward and accuracy logged every $100$ steps, which provide the eval-side
trajectories used both for the three-phase fit and for the reward-hacking
(train--eval co-saturation) diagnostic. \emph{Out-of-distribution}, we measure
zero-shot transfer of the GSM8K-trained policies---without any further
tuning---to eight external math benchmarks spanning the full difficulty range,
from grade-school word problems to olympiad competition mathematics:
GSM-Plus \citep{li2024gsmplus}, MetaMathQA \citep{yu2024metamath},
OpenMath2 \citep{toshniwal2024openmath2}, NuminaMath \citep{numina_math_datasets},
MATH-500 \citep{hendrycks2021math, lightman2023letsverify},
IMO-Bench, AIME-2026, and HMMT-2025. The benchmarks are ordered, roughly, by
distance from the GSM8K training distribution, which lets us read transfer as a
function of that distance.

\paragraph{Metric.} Transfer is reported as pass@$k$ for $k\in\{1,5\}$, the
fraction of problems solved within $k$ sampled completions, as mean $\pm$
standard deviation over seeds. Reporting both $k=1$ and $k=5$ guards against
mistaking a temperature-induced broadening of the output distribution for a
genuine accuracy gain: an improvement that reflects real capability raises
pass@1 and not merely pass@5. For the larger benchmarks we subsample for
tractability---$2000$ examples each from MetaMathQA and OpenMath2---and for
NuminaMath we skip the single proof problem that is not objectively verifiable;
benchmark sizes are noted in Table~\ref{tab:gsm8k_benchmark}.

\subsection{Curve-Fitting Methodology}
\label{subsec:exp_fitting}

To test the closed-form prediction of Theorem~\ref{thm:main} we fit each
seed-averaged trajectory with the critically-damped three-phase law of
Proposition~\ref{prop:threephase},
\[
  R(t) \;=\; R_\infty - (R_\infty - R_0)\,(1 + \omega_n t)\,e^{-\omega_n t},
\]
a three-parameter family in the initial reward $R_0$, the plateau
$R_\infty=p^\star$, and the natural frequency $\omega_n=\sqrt{K_{\mathrm{eff}}/M_{\mathrm{eff}}}$.
Fits are obtained by nonlinear least squares against the seed-averaged curve, and
goodness of fit is reported as the in-panel coefficient of determination $R^2$
measured on that curve. The same fitted $(R_0,R_\infty)$ are reused to form the
scale- and offset-free progress variable
$u(t)=(R(t)-R_0)/(R_\infty-R_0)$ that underlies the doubly-normalized
co-saturation diagnostic. Because the deterministic coefficients are
$G$-independent by Proposition~\ref{prop:group}, a single functional form is fit
to both group sizes, and agreement of the fitted $R_\infty$ and shape across
$G\in\{4,16\}$ is itself one of the predictions under test.

\section{Results}
\label{sec:results}
\subsection{Description of GRPO Training Dynamics}
\label{subsec:gpro_training_dynamics}
We train three models---Nemotron-Mini-4B, DeepSeek-LLM-7B-Chat, and
DeepSeek-R1-Distill-Qwen-1.5B---with GRPO on GSM8K under two group sizes
$G \in \{4, 16\}$, using the configuration of Table~\ref{tab:gsm8k_config} in Appendix and
averaging over seeds $\{42, 3407\}$. For each run we log the training reward,
the held-out evaluation reward, and the evaluation accuracy as functions of the
optimizer step, and we fit each trajectory with the critically-damped
three-phase law of Proposition~\ref{prop:threephase},
\[
  R(t) \;=\; R_\infty - (R_\infty - R_0)\,(1 + \omega_n t)\,e^{-\omega_n t},
\]
a three-parameter family in $(R_0, R_\infty, \omega_n)$. Figure~\ref{fig:grpo_dynamics_gsm8k} collects the resulting curves with the in-panel coefficients of determination.

The central observation is that the \emph{training}-reward trajectories are
captured by the reduced-order model to high accuracy: across all three models and both group sizes the fit attains $R^2 \ge 0.91$ (panels (a)--(f)), reaching $R^2 = 0.97$ for the Nemotron and Qwen-1.5B runs. The three regimes predicted by Eq.~\ref{eq:collapse} are visible in every train panel: (i)~a slow start in which the reward changes little---the signature of the zero-velocity initial condition $R'(0)=0$ enforced by the momentum-plus-lag mass $M_{\mathrm{eff}}$ of Lemma~\ref{lem:lag}; (ii)~a rapid phase centered on the inflection $t^\star = 1/\omega_n$; and (iii)~a plateau at $R_\infty = p^\star$, the GRPO fixed point of Assumption~\ref{ass:mf}. The single-exponential saturation
law in Eq.~\ref{eq:emp}, being concave from $t=0$ with maximal initial slope, cannot represent the slow start; it is precisely the retained inertia that supplies the early inflection and lifts the fit to the observed quality.

Increasing $G$ from $4$ to $16$ leaves the fitted plateau $R_\infty$ and the qualitative slow-start/rapid-phase/plateau shape essentially unchanged while visibly narrowing the run-to-run scatter (the shaded $\pm$s.d.\ band contracts at $G=16$). This is the empirical content of the deterministic $G$-invariance of the coefficients Eq.~\ref{prop:group}: $G$ enters the dynamics only through the forcing intensity $D/G$, so it sets the stationary fluctuation$\langle e^2\rangle_\infty \propto 1/G$ rather than the mean trajectory. The two rollout budgets therefore trace the same mean curve at different noise levels, consistent with reports that two rollouts per prompt match sixteen up to noise.

The evaluation-side metrics are more nuanced. At $G=16$ the held-out reward is fit well (panels (j)--(l), $R^2$ from $0.92$ to $0.98$) and the evaluation accuracy improves substantially (panels (p)--(r), $R^2$ from $0.75$ to $0.95$), confirming that the train-time order parameter $p = \mathbb{E}[R]$ tracks the held-out quantities through saturation. At $G=4$, by contrast, the eval-reward and eval-accuracy fits degrade markedly (panels (g)--(i) and (m)--(n), $R^2 \approx 0.55$--$0.69$). We read this as a logging artifact rather than a breakdown of the dynamics: evaluation is recorded only every $100$ steps
(Table~\ref{tab:gsm8k_config} in Appendix), and at the smaller group size the seed-averaged eval curve is both shorter and noisier, leaving too few points to pin a three-parameter fit. The $G=16$ panels, which span more steps and average over more rollouts per step, recover the fit on the same quantities. The lone exception at $G=4$ is Qwen-1.5B, whose eval-accuracy fit is already high ($R^2 = 0.97$, panel (o))---the model that also shows the cleanest reward--eval coupling in our diagnostics.

Taken together, Figure~\ref{fig:grpo_dynamics_gsm8k} substantiates the empirical claim of Section~\ref{sec:formalism}: within the stable, mean-field regime, GRPO reward trajectories are not merely monotone but follow the specific closed-form shape of an overdamped-to-critically-damped second-order relaxation toward a fixed point, with the deterministic part of that shape independent of group size. 

\begin{figure}[!htbp]
    \centering    \includegraphics[width=\linewidth,height=0.65\textheight,keepaspectratio]{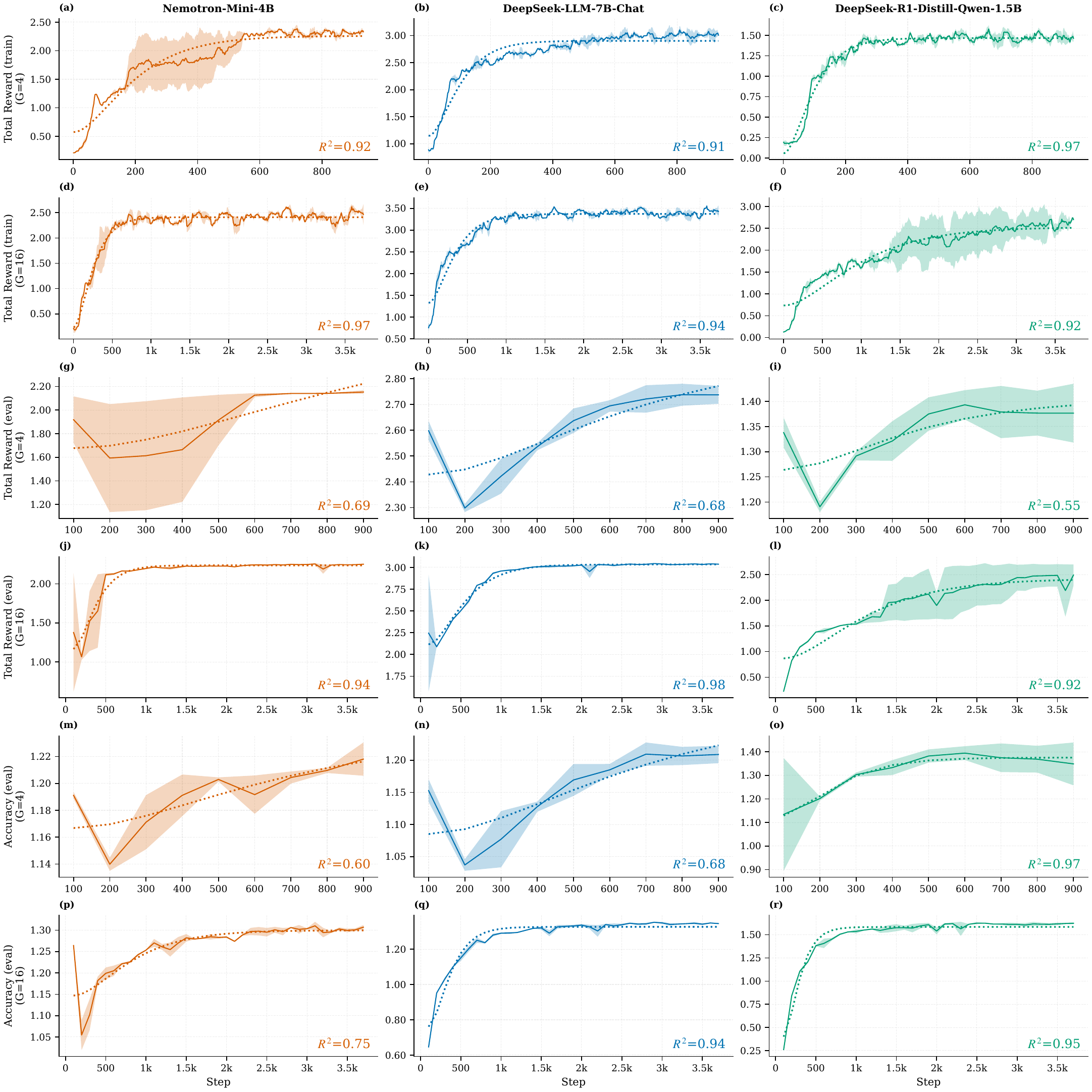}
    \caption{%
  GRPO training dynamics on GSM8K for three models (columns:
  Nemotron-Mini-4B, DeepSeek-LLM-7B-Chat, DeepSeek-R1-Distill-Qwen-1.5B)
  across two group sizes $G\in\{4,16\}$. Rows show total reward on train
  (rows 1--2) and eval (rows 3--4), and eval accuracy (rows 5--6), each
  pair organized as $G{=}4$ above $G{=}16$. Solid curves are seed-averaged
  means ($\pm$ s.d.\ shaded band) over seeds $\{42, 3407\}$; dotted curves
  are the critically-damped three-phase fit
  $R(t)=R_\infty-(R_\infty-R_0)(1+\omega_n t)\,e^{-\omega_n t}$ from
  Prop.~\ref{prop:threephase}, with the in-panel $R^2$ measured against
  the seed-averaged curve. Train-reward panels are fit by Eq.~\eqref{eq:threephase}
  to $R^2\ge 0.91$ across all models and group sizes; eval-side fits at
  $G{=}4$ degrade ($R^2\!\approx\!0.55$--$0.69$) because eval logging is
  sparse at the small group size, while the corresponding $G{=}16$ panels
  recover $R^2\ge 0.92$. Consistent with
  Proposition~\ref{prop:group}, the deterministic trajectory is
  approximately $G$-invariant: doubling $G$ from 4 to 16 leaves the
  fitted $R_\infty$ and the qualitative slow-start/rapid-phase/plateau
  shape unchanged, while reducing run-to-run scatter.%
  }
    \label{fig:grpo_dynamics_gsm8k}
\end{figure}

\subsection{Reward Hacking Investigation}

Reward hacking is the regime in which the optimizer continues to drive the
training reward toward its plateau while the held-out objective fails to follow:
the policy games features of the training signal that do not transfer. If the train and eval rewards both track the same scalar order parameter $p=\mathbb{E}[R]$ relaxing to
$R_\infty=p^\star$, then by the saturation law (Eq.~\ref{eq:sat}) each obeys
$-\log\!\bigl(R_\infty - R(t)\bigr) = \lambda t - \log(R_\infty - R_0)$, a line
whose slope is the shared relaxation rate $\lambda$. A healthy run is therefore
one in which train and eval \emph{co-saturate}---same shape, same rate---and
reward hacking is the failure of that co-saturation. Figure~\ref{fig:gsm8k-hacking}
makes this operational through two complementary read-outs.

\paragraph{Doubly-normalized plateau ratio (row~1).} For each metric we form the
scale- and offset-free progress variable
$u(t) = \bigl(R(t) - R_0\bigr)/\bigl(R_\infty - R_0\bigr)$ using the fitted
$(R_0, R_\infty)$ of Proposition~\ref{prop:threephase}, and plot $u_{\mathrm{eval}}(t)$
against $u_{\mathrm{train}}(t)$ parametrically (panels (a)--(c), star at the final
step). The normalization removes the reward scale and the initial offset, so a
run that co-saturates falls on the diagonal reference $y=x$, regardless of how
fast it trains or how high its plateau sits. A run whose eval progress lags its
train progress---reward rising on the training distribution faster than it
generalizes---bends \emph{below} the diagonal, and the vertical gap at a given
$u_{\mathrm{train}}$ is a direct, parameter-free measure of the hacking deficit.

\paragraph{Rolling coupling ratio (row~2).} The parametric view summarizes the
whole run but obscures \emph{when} decoupling sets in. We therefore estimate the
local rates by regressing $-\log\!\bigl(R_\infty - R(t)\bigr)$ on the step within
a rolling window (45\% of the eval-step span), separately for the train and eval
points, and report their ratio
\[
  \rho(t) = \frac{\lambda_{\mathrm{eval}}(t)}{\lvert\lambda_{\mathrm{train}}(t)\rvert}.
\]
Co-saturation is $\rho \approx 1$; the hacking signature is $\rho \to 0$, the
eval rate collapsing relative to the train rate. We plot $\rho$ on a
symmetric-log axis ($\mathrm{linthresh}=1$) against training progress normalized
per $G$, so the $G=4$ and $G=16$ traces overlay and the healthy band near unity
renders linearly.

\paragraph{Findings.} Three patterns emerge. (i)~At $G=4$, Nemotron-Mini-4B and
DeepSeek-LLM-7B-Chat both display the hacking signature in both read-outs: their
parametric curves (panels (a),(b)) sit well below the diagonal, and the coupling
ratio (panels (d),(e)) stays below one throughout training, so the held-out
reward never matches the pace of the training reward. (ii)~DeepSeek-R1-Distill-
Qwen-1.5B at $G=16$ (panel (f), purple) is the cleanest healthy run, with
$\rho \approx 1$ over the middle $80\%$ of training; its parametric curve
correspondingly hugs the diagonal. This is consistent with Section~\ref{subsec:gpro_training_dynamics},
where Qwen-1.5B was the one model whose eval-accuracy trajectory was already
well described by the reduced-order fit at $G=4$ (panel~(o) of
Figure~\ref{fig:grpo_dynamics_gsm8k}): the cleanest order-parameter coupling here and the
cleanest fit there are the same phenomenon. (iii)~The two read-outs agree on the
direction of the group-size effect---larger $G$ pushes runs toward the diagonal
and $\rho$ toward unity---so the rollout budget that Proposition~\ref{prop:group}
predicts to leave the \emph{mean} train trajectory unchanged nonetheless improves
the train--eval coupling, an effect outside the scalar reduction and attributable
to the reduced gradient noise at larger $G$.

One reading caveat accompanies the coupling ratio. Late in training, as
$R_{\mathrm{train}} \to R_\infty$, the train rate $\lambda_{\mathrm{train}} \to 0$
and $\rho$ becomes ill-defined; the resulting symlog spikes in panels (d) and (e)
are numerical artifacts of train saturation, not a new hacking signal, and the
$\mathrm{linthresh}=1$ axis is chosen precisely so that they are visible but
compressed. The informative portion of $\rho$ is therefore the interior of
training, where $\lambda_{\mathrm{train}}$ is bounded away from zero. Read this
way, Figure~\ref{fig:gsm8k-hacking} locates the training fraction at which the held-out
objective decouples from the training reward---complementing the diagnostics of
mode collapse and stability, which instead test where the
mean-field reduction itself, rather than its transfer, begins to fail.

\begin{figure}[!htbp]
    \centering    \includegraphics[width=\linewidth,height=0.65\textheight,keepaspectratio]{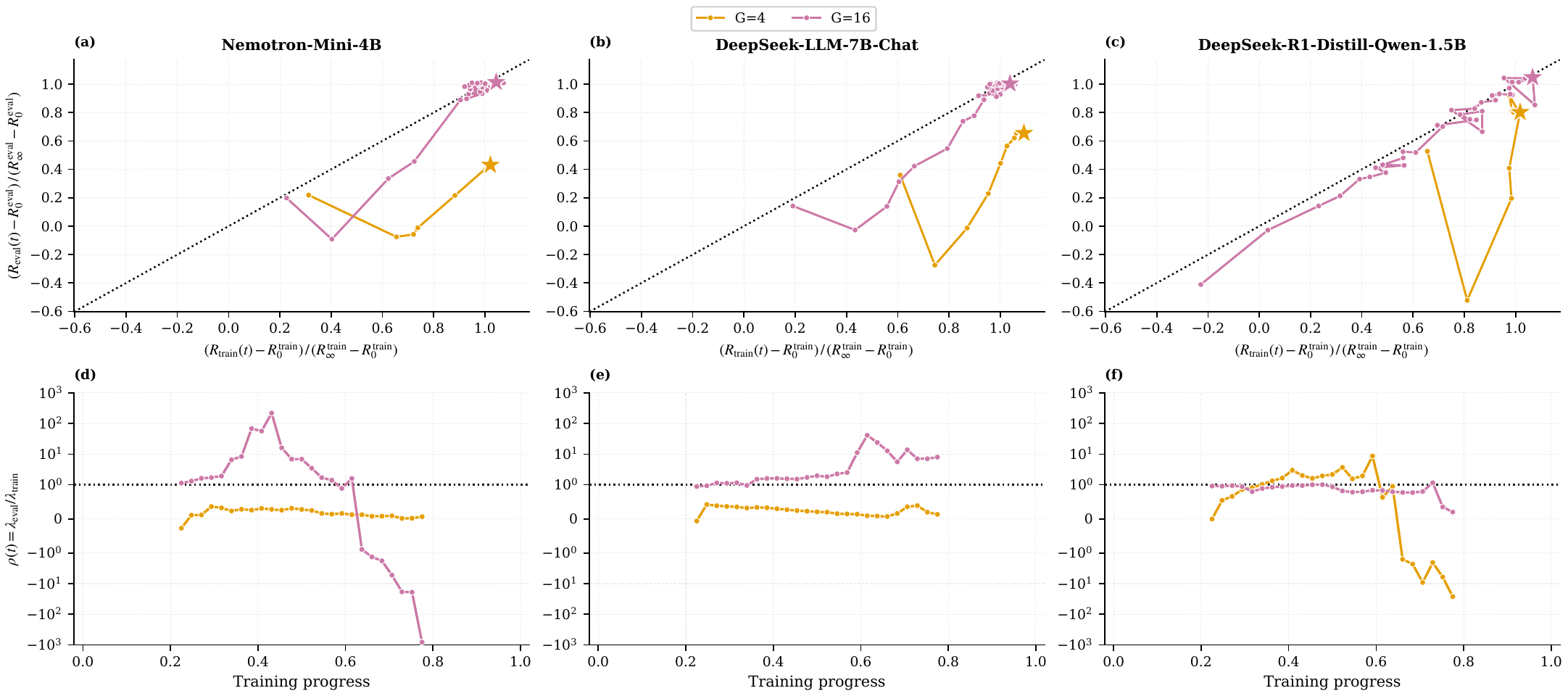}
    \caption{%
  Reward-hacking diagnostic for GRPO on GSM8K, three models (columns:
  Nemotron-Mini-4B, DeepSeek-LLM-7B-Chat, DeepSeek-R1-Distill-Qwen-1.5B),
  two group sizes ($G=4$ orange, $G=16$ purple).
  \textbf{Row 1, panels (a)--(c) -- doubly-normalized plateau ratio.}
  Each curve is the parametric trajectory of
  $\bigl(R_\text{train}(t)-R_0^\text{train}\bigr)/\bigl(R_\infty^\text{train}-R_0^\text{train}\bigr)$
  on the $x$-axis against the same normalization of $R_\text{eval}(t)$
  on the $y$-axis, with $R_0, R_\infty$ from the critically-damped fit
  of Prop.~\ref{prop:threephase} and a star marking the final training
  step. The dotted black line is the healthy reference $y=x$: a
  co-saturating run hugs the diagonal; a hacking run bends below it.
  \textbf{Row 2, panels (d)--(f) -- rolling coupling ratio
  $\rho(t) = \lambda_\text{eval}(t)/|\lambda_\text{train}(t)|$.}
  Local $\lambda$ in a rolling window (45\% of the eval-step span) is
  the slope of $-\log\bigl(R_\infty-R(t)\bigr)$ against step, regressed
  separately on train and eval points; $\rho \approx 1$ is healthy
  coupling, $\rho \to 0$ is the hacking signature
  (Sec.~\ref{sec:saturation}). The $x$-axis is normalized training
  progress so $G=4$ and $G=16$ overlay; the $y$-axis is symmetric-log
  with $\text{linthresh}=1$, so the $\rho \approx 1$ regime is linear
  and the late-training blow-ups (where $R_\text{train}\to R_\infty$
  makes $\lambda_\text{train}\to 0$ and $\rho$ ill-defined) are visible
  but visibly compressed.
  \textbf{Findings.} (i) $G=4$ shows the hacking signature for
  Nemotron and DeepSeek-7B in both rows: parametric curves in (a) and
  (b) sit well below the diagonal and $\rho$ in (d) and (e) stays
  below 1 throughout training. (ii) Qwen-1.5B at $G=16$ (panel (f),
  purple) is the cleanest healthy run: $\rho \approx 1$ over the
  middle 80\% of training. (iii) Late-training symlog spikes in (d)
  and (e) are numerical artifacts of train saturation, not new
  hacking signal.%
  }
  \label{fig:gsm8k-hacking}
\end{figure}

\subsection{Mode Collapse Investigation}

The reduced-order model treats the group size purely as a temperature: by
Lemma~\ref{lem:noise} the only role of $G$ is to set the forcing intensity
$D \propto \sigma_{\mathrm{adv}}^2/G$, and by Proposition~\ref{prop:group}
this forcing enters Eq.~\eqref{eq:floor} without touching the deterministic
coefficients. Mode collapse threatens this picture in two distinct ways, and
neither is a failure of the \emph{dynamics}---both are failures of the
\emph{construction} on which the dynamics rest. First, the within-group reward
spread $\sigma_{\mathrm{adv}}$ can itself drift or acquire a $G$-dependence that
Lemma~\ref{lem:noise} does not model, in which case the temperature
interpretation is no longer clean. Second, and more severely, the advantage
construction can degenerate: when all $G$ completions of a prompt receive the
same scalar reward, $\sigma_{\mathrm{adv}}=0$ and the GRPO advantage
$A_i=(r_i-\bar r)/\sigma_{\mathrm{adv}}$ of Eq.~\ref{eq:adv} becomes the
indeterminate $0/0$, so the prompt contributes no usable signal at all.
Figure~\ref{fig:gsm8k-collapse-diag} tests for each in turn, with training progress
normalized per $G$ so the two group sizes overlay on the $x$-axis.

\paragraph{Noise-intensity invariance test (row~1).} We plot the seed-averaged
within-group advantage variance $\sigma_{\mathrm{adv}}^2/G$ for both group sizes
(panels (a)--(c)), where $\sigma_{\mathrm{adv}}$ is the logged train-reward
standard deviation. This is precisely the forcing intensity of
Lemma~\ref{lem:noise}, so under the paper's assumptions the $G=4$ and $G=16$
traces should sit on a common, flat plateau: the temperature is the only thing
$G$ controls, and once it is divided out the curves coincide. Two departures are
diagnostic. A \emph{downward drift} away from the plateau means the gradient
signal in Eq.~\ref{eq:pg} is shrinking faster than the $1/G$ noise
reduction can account for---within-group rewards are losing their variability as
the policy concentrates mass on a shrinking set of completions. A
\emph{separation} of the two curves means $\sigma_{\mathrm{adv}}$ carries a
$G$-dependence of its own, again outside what Lemma~\ref{lem:noise} predicts. A
persistent, overlapping plateau is the healthy reading.

\paragraph{Collapsed-prompt fraction (row~2).} We plot the seed-averaged
fraction of prompts in a training step whose $G$ completions all received an
identical scalar reward (panels (d)--(f)). For these prompts the advantage is
the indeterminate $0/0$ of Eq.~\ref{eq:pg}, so they drop out of the
update entirely. A fraction near zero is healthy; a monotone climb toward one is
the empirical witness for the regime in which Lemma~\ref{lem:noise} and
Proposition~\ref{prop:group} cease to apply, not because the relaxation
of Theorem~\ref{thm:main} has been violated but because the object it relaxes
has stopped being well defined.

\paragraph{Findings.} The combined signature of a faithful run is a flat,
overlapping plateau in row~1 together with a collapsed-prompt fraction near zero
in row~2, and Figure~\ref{fig:gsm8k-collapse-diag} shows that the early portion of training sits in this regime across models: the advantage variance is on a stable
plateau and almost every prompt yields a usable advantage. Both diagnostics then
develop deviations as training proceeds. The row-1 traces do not collapse onto a
single curve to the extent the pure-temperature account requires---the residual
gap between $G=4$ and $G=16$ indicates a $G$-dependence of $\sigma_{\mathrm{adv}}$
beyond the $1/G$ scaling of Lemma~\ref{lem:noise}---and the collapsed-prompt
fraction in row~2 grows away from zero rather than holding, most visibly for
DeepSeek-LLM-7B-Chat, where a substantial share of prompts eventually carries no
gradient signal. These are concentration signatures: the policy is committing to
fewer completions per prompt, which simultaneously thins $\sigma_{\mathrm{adv}}$
(row~1) and pushes whole prompts into the $0/0$ degeneracy (row~2).

\paragraph{Scope.} The value of this diagnostic is that it separates two failure
modes that a reward curve alone cannot distinguish. A plateau in reward could
mean the order parameter has reached $p^\star$ (Theorem~\ref{thm:main}
holding) or that the advantage construction has degenerated; rows~1 and~2 tell these apart by inspecting the variability of the
signal rather than its level. Where row~1 drifts down or row~2 climbs, the
mean-field reduction of Assumption~\ref{ass:mf} is being eroded by mode
concentration---the policy collapsing onto a small set of completions---rather
than by any violation of the second-order relaxation itself. 

\begin{figure}[!htbp]
    \centering    \includegraphics[width=\linewidth,height=0.65\textheight,keepaspectratio]{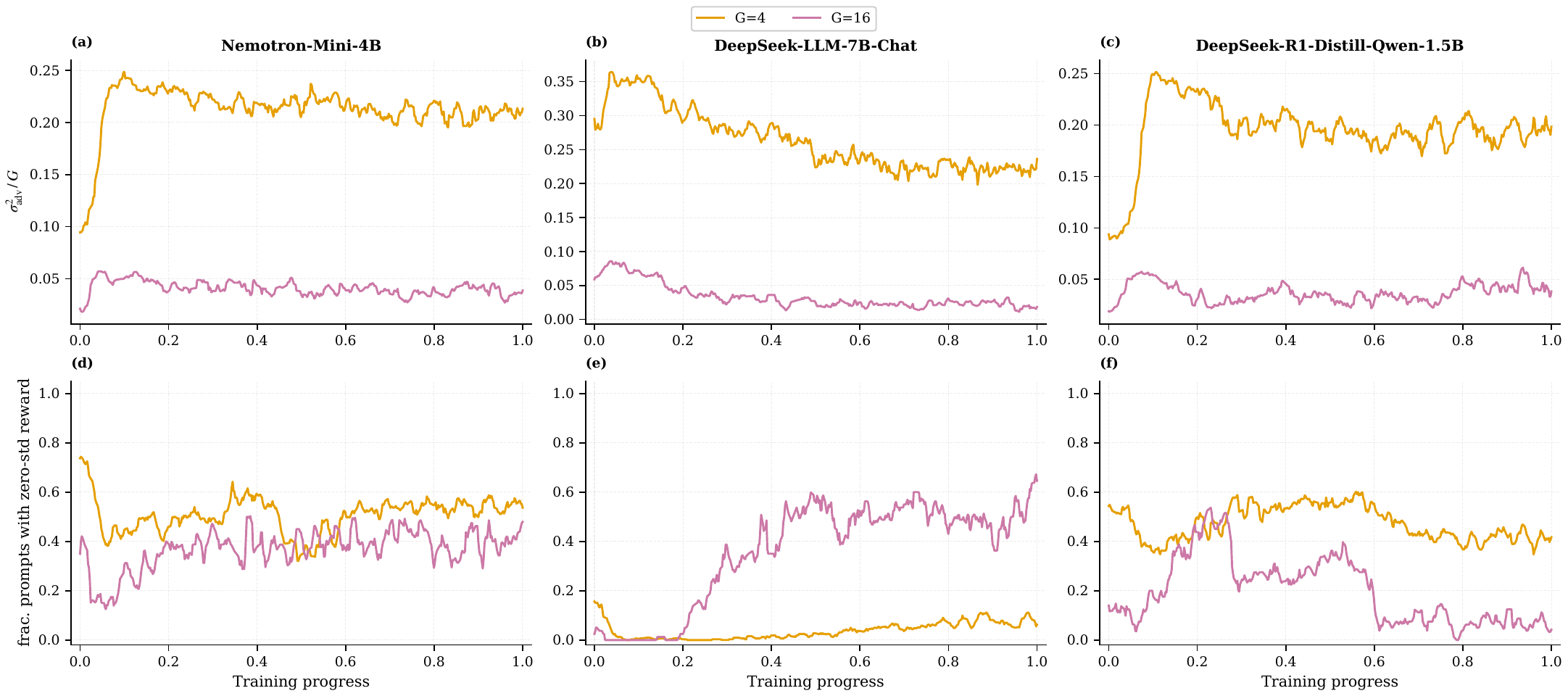}
   \caption{%
  Mode-collapse diagnostic for GRPO on GSM8K, three models (columns:
  Nemotron-Mini-4B, DeepSeek-LLM-7B-Chat, DeepSeek-R1-Distill-Qwen-1.5B),
  two group sizes ($G=4$ orange, $G=16$ purple). The $x$-axis is
  training progress normalized per $G$ so the two curves overlay.
  \textbf{Row 1, panels (a)--(c) -- noise-intensity invariance test.}
  Each curve is the seed-averaged within-group advantage variance
  $\sigma_\mathrm{adv}^2/G$, where $\sigma_\mathrm{adv}$ is the logged
  \texttt{train/reward\_std}. By Lemma~\ref{lem:noise} the GRPO noise
  intensity is $D\propto\sigma_\mathrm{adv}^2/G$; by
  Proposition~\ref{prop:group} $G$ enters Eq.~\eqref{eq:main} only
  through this forcing intensity, so under the paper's assumptions
  the $G=4$ and $G=16$ curves should track each other on a common
  plateau. Departure from that plateau (downward drift) indicates
  the gradient signal in Eq.~\eqref{eq:adv} is shrinking faster than
  the $1/G$ noise reduction, i.e., the policy is concentrating mass
  and within-group rewards lose variability. Separation of the two
  curves indicates a $G$-dependence of $\sigma_\mathrm{adv}$ that
  Lemma~\ref{lem:noise} does not predict.
  \textbf{Row 2, panels (d)--(f) -- collapsed-prompt fraction.}
  Seed-averaged \texttt{train/frac\_reward\_zero\_std}: the fraction
  of prompts in a training step whose $G$ completions all received
  the same scalar reward, so the GRPO advantage
  $A_i=(r_i-\bar r)/\sigma_\mathrm{adv}$ in Eq.~\eqref{eq:adv} is the
  indeterminate $0/0$. Zero is healthy; a monotone climb toward one
  is the empirical witness for the regime where Lemma~\ref{lem:noise}
  and Proposition~\ref{prop:group} cease to apply because the
  advantage construction itself has degenerated.
  \textbf{Reading.} Persistent flat overlap of the two $G$ curves in
  row~1, combined with row~2 near zero, is the signature of a
  training run where the reduced-order dynamics of
  Theorem~\ref{thm:main} remain a faithful description. A drop in
  row~1 or a rise in row~2 marks the training step at which the
  mean-field reduction (Assumption~\ref{ass:mf}) is being eroded by
  mode concentration rather than by a violation of the dynamics.%
  }
  \label{fig:gsm8k-collapse-diag}
\end{figure}

\subsection{Policy Concentration Investigation}

Of all the steps behind Theorem~\ref{thm:main}, only Assumption~\ref{ass:mf}
is non-mechanical (Remark~\ref{rem:scope}): it posits that the entire policy is
summarized to leading order by the scalar order parameter $p(\theta)=\mathbb{E}[R]$.
This is a statement about the \emph{dimensionality} of the dynamics, not about
the reward level. It can fail even while the reward evolves smoothly, if the
policy quietly collapses its mass onto a small set of completions: $p(\theta)$
continues along a one-dimensional path while the underlying high-dimensional
$\theta$-dynamics are no longer captured by it. This policy concentration is the
direct threat to Assumption~\ref{ass:mf}, and it is distinct from the
advantage degeneracy---there the construction in
Eq.~\ref{eq:obj} broke down, here the reduction's central premise does.
Figure~\ref{fig:gsm8k-concentration} probes it through two observables: the token-level
policy entropy, which measures how spread the policy remains, and the KL anchor,
which measures how far it has travelled from the reference.

\paragraph{Token-level policy entropy (row~1).} We plot the seed-averaged
per-token entropy of $\pi_\theta(\cdot\mid q)$ averaged across the batch (panels
(a)--(c)). A high, stable entropy is the regime in which the scalar reduction is
defensible: the policy retains its breadth, so summarizing it by $p(\theta)$
discards little. A monotone decrease is the empirical signature that the
reduction is becoming a poor description of the $\theta$-dynamics, since mass is
concentrating on a few completions while $p(\theta)$ may still be moving
smoothly; a rapid drop to a low plateau is the concentration signal proper.
Because Proposition~\ref{prop:group} makes the deterministic dynamics
$G$-invariant, the entropy trajectories should additionally \emph{overlap}
across group sizes; curves that do not overlap indicate a $G$-driven
concentration that the reduction does not predict.

\paragraph{KL anchor (row~2).} We plot the seed-averaged
$D_{\mathrm{KL}}(\pi_\theta \,\|\, \pi_{\mathrm{ref}})$ (panels (d)--(f)), the
term carrying weight $\beta$ in the objective Eq.~\ref{eq:obj} and supplying
the restoring force of Lemma~\ref{lem:backbone}. The KL alone is ambiguous, and
is informative only when read against row~1. Rising KL \emph{with} falling
entropy is the concentration regime---the policy has both moved far from the
reference and lost breadth, the mass loss being what reduced the entropy. Rising
KL with \emph{stable} entropy is the benign case: a pure policy shift, a
translation in policy space that the order-parameter description can absorb
without the scalar summary failing.

\paragraph{Findings.} Every run shows entropy declining and KL rising, which is
expected as GRPO pulls the policy away from its initialization; the question is
whether the joint motion is the benign shift or the concentration failure, and
the three models separate cleanly on this axis. DeepSeek-LLM-7B-Chat shows the
strongest concentration signature: its KL rises to the largest values of the
three models and its entropy falls correspondingly, the falling-entropy/rising-KL
combination that row~2 is designed to flag. Qwen-1.5B sits at the opposite
extreme---its KL stays small, so the policy remains close to the reference and
its entropy declines least---marking it as the run for which
Assumption~\ref{ass:mf} remains most faithful. Nemotron-Mini-4B is
intermediate. This ordering is the same one recovered by the earlier diagnostics:
Qwen-1.5B was the model with the cleanest order-parameter fit and the cleanest
train--eval coupling, while DeepSeek-7B was the model whose advantage construction degenerated most. The entropy trajectories are moreover not perfectly $G$-invariant---the $G=4$ and $G=16$ curves do not coincide as cleanly as Proposition~\ref{prop:group} would require---which is the same $G$-dependence of the within-group statistics seen earlier, now visible at the level of the policy itself.

\paragraph{Scope.} The decisive quantity is the \emph{timing} of the entropy drop
relative to reward saturation. An entropy decline that occurs only after the
reward has reached its plateau (Figure~\ref{fig:grpo_dynamics_gsm8k}) is comparatively
harmless: by then $p \approx p^\star$ and the reduction has already done its
work. An entropy drop that precedes saturation is the dangerous case, because the
reward---and hence $p(\theta)$---is still evolving at a point where the policy
has already lost its dimensionality, so Theorem~\ref{thm:main} should no
longer be expected to describe the trajectory. A nearly flat entropy curve with a
slowly rising KL therefore certifies a run in which the scalar reduction holds
throughout, whereas a sharp pre-saturation entropy collapse locates the training
fraction at which it ceases to. Together with the stability diagnostic, which tests the complementary failure mode in which
the dynamics themselves go unstable, this completes the account of where the
closed-form description of Section~\ref{sec:formalism} can and cannot be trusted.

\begin{figure}[!htbp]
\centering
\includegraphics[
  width=\linewidth,
  height=0.60\textheight,
  keepaspectratio
]{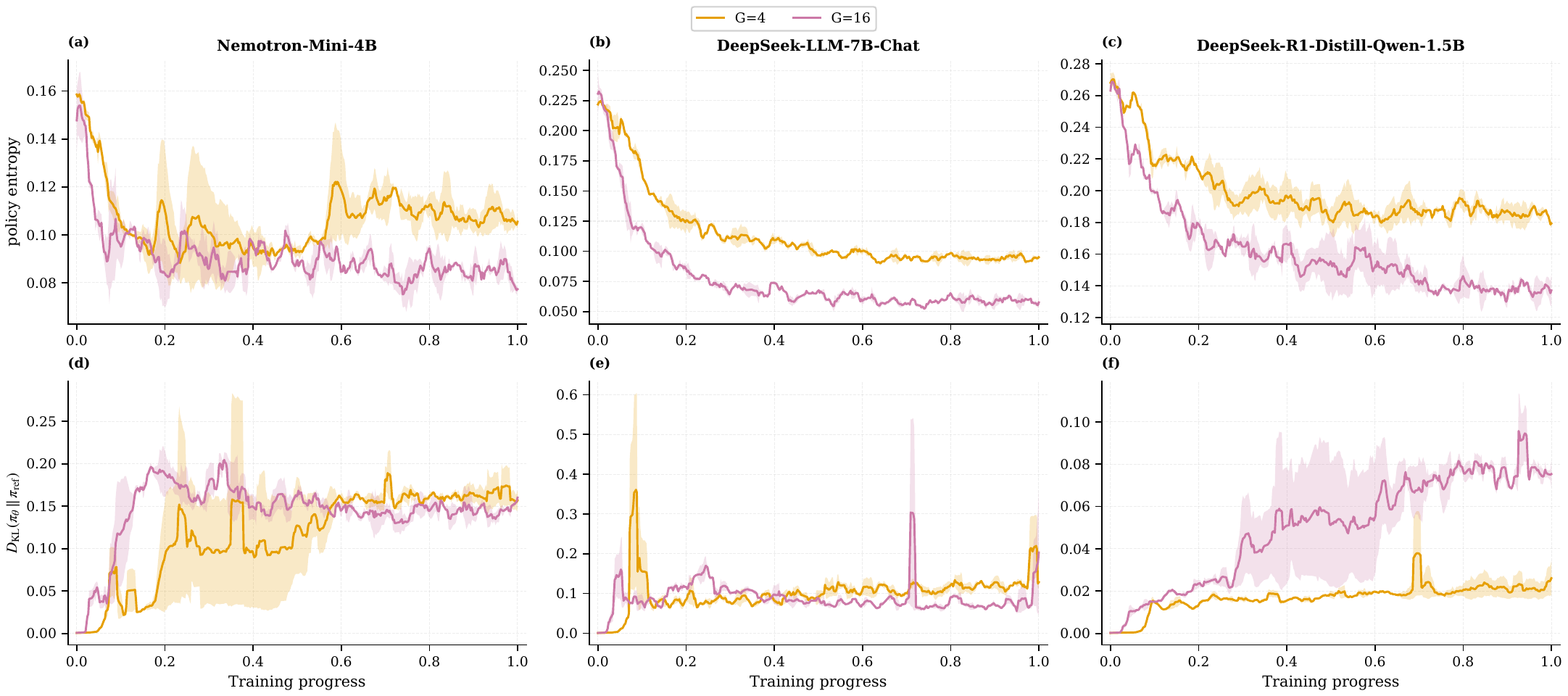}
\caption{Policy-concentration diagnostics for GRPO on GSM8K. Columns correspond
to Nemotron-Mini-4B, DeepSeek-LLM-7B-Chat, and DeepSeek-R1-Distill-Qwen-1.5B.
Curves show two group sizes ($G=4$, orange; $G=16$, purple). The $x$-axis is
training progress normalized by group size. Solid lines denote seed-averaged
means over seeds $\{42,3407\}$ and shaded regions indicate $\pm$ one standard
deviation. \textbf{Top row (a)--(c):} Token-level policy entropy
(\texttt{train/entropy}), averaged across generated tokens and prompts. Declining
entropy indicates increasing concentration of probability mass on a smaller set
of completions. \textbf{Bottom row (d)--(f):} KL divergence from the reference
policy, \texttt{train/kl} $= D_{\mathrm{KL}}(\pi_\theta \Vert \pi_{\mathrm{ref}})$.
Rising KL together with falling entropy indicates policy concentration, whereas
rising KL with approximately stable entropy corresponds to a broader policy
shift. Together, these diagnostics assess whether the scalar order-parameter
assumption remains a faithful description of the underlying policy dynamics during
training.}
\label{fig:gsm8k-concentration}
\end{figure}

\subsection{Training Stability Investigation}

Theorem~\ref{thm:main} can fail in two ways, first: the reduction's premise eroding while the dynamics themselves
remain well behaved. This section tests the complementary failure—the dynamics becoming unstable while the reduction's premise still holds. The linearized system Eq.~\ref{eq:main} is unstable iff $K\eta K_{\mathrm{ref}} > 1-\mu$, at which point the effective damping
$C_{\mathrm{eff}}$ turns negative and the characteristic
roots cross into the right half-plane. Figure~\ref{fig:gsm8k-stability} tracks this
through two observables that stand in a cause-and-effect relation: the importance-sampling ratio, which witnesses the off-policy lag that erodes the damping, and the gradient norm, which is where a loss of damping shows up as divergence.

\paragraph{Gradient-norm trajectory (row~1).} We plot the seed-averaged $\|\nabla J\|$ on a logarithmic $y$-axis (panels (a)--(c)). When
Proposition~\ref{prop:stab} is violated and $C_{\mathrm{eff}}<0$, the
right-half-plane roots produce monotone exponential growth---a straight ascending
line on a log scale---or abrupt spikes as the gradient blows up. A bounded,
weakly oscillating trace is the converse signature: the roots remain in the left
half-plane and the linearized model of Theorem~\ref{thm:main} continues to
apply.

\paragraph{Importance-sampling drift (row~2).} We plot the seed-averaged maximum
over the batch of $\pi_\theta(o\mid q)/\pi_{\mathrm{old}}(o\mid q)$ (panels
(d)--(f)), where $\pi_{\mathrm{old}}=\pi_{\theta_{k-K_{\mathrm{ref}}}}$ is the
refreshed sampling policy of Eq.~\eqref{eq:update}. This is the direct witness of
the off-policy lag of Lemma~\ref{lem:lag}: as the delay $\tau = K_{\mathrm{ref}}\delta$
grows, $\pi_{\mathrm{old}}$ and $\pi_\theta$ separate, the support of the ratio
widens, and the $-K\tau$ damping reduction in $C_{\mathrm{eff}}=c_\mu - K\tau$
becomes material. A blow-up of the IS-ratio max is therefore the practical
precursor of crossing the threshold~\eqref{eq:threshold}: row~2 supplies the
mechanism, row~1 the consequence, so widening in row~2 should lead any divergence
in row~1.

\paragraph{Findings.} Across all three models and both group sizes the training
sits squarely in the stable regime. The gradient norms (row~1) remain bounded and
weakly oscillating on the log scale, with no monotone exponential growth and no
runaway spikes---the signature of left-half-plane roots throughout training. The
IS-ratio maxima (row~2) settle onto a bounded plateau, fluctuating in a narrow
band (roughly $2.3$--$2.9$ across the three models) rather than diverging. The
off-policy lag is thus clearly present---the max ratio sits well above one, so
$\pi_{\mathrm{old}}$ and $\pi_\theta$ do separate between refreshes---but its
damping cost $-K\tau$ is finite and never large enough to drive $C_{\mathrm{eff}}$
negative, so the stable condition $K\eta K_{\mathrm{ref}} < 1-\mu$ holds for the
duration of every run. Mechanistically this is unsurprising: the small learning
rate in Table~\ref{tab:gsm8k_config} in Appendix keeps the product $K\eta K_{\mathrm{ref}}$ well
below $1-\mu$, which is precisely the overdamped, small-$\eta$ regime that
Proposition~\ref{prop:zeta} and Remark~\ref{rem:corr} identify
with monotone saturation. Both diagnostics also overlay across group sizes,
consistent with the $G$-freeness of the deterministic coefficients in
Proposition~\ref{prop:group}.

\begin{figure}[!htbp]
    \centering    \includegraphics[width=\linewidth,height=0.65\textheight,keepaspectratio]{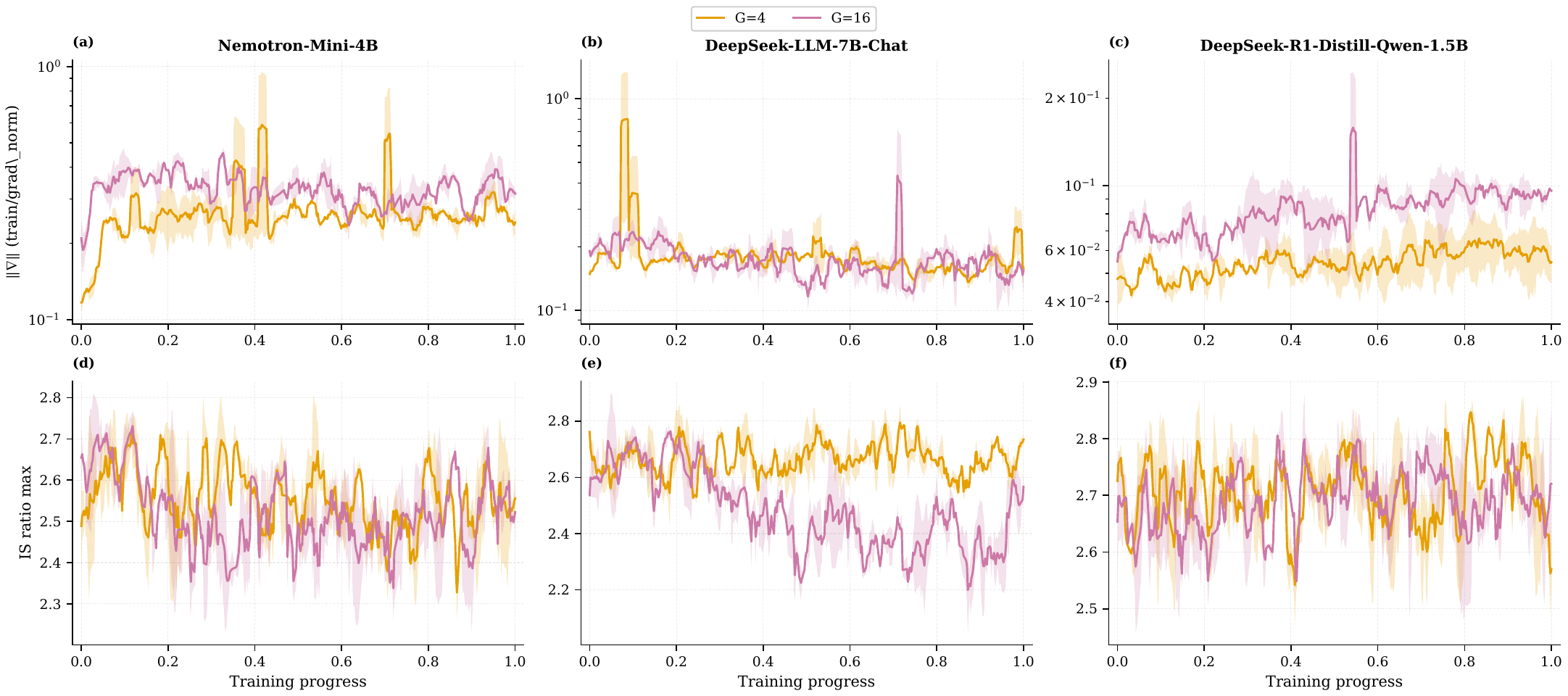}
\caption{%
  Stability diagnostic for GRPO on GSM8K, three models (columns:
  Nemotron-Mini-4B, DeepSeek-LLM-7B-Chat,
  DeepSeek-R1-Distill-Qwen-1.5B), two group sizes ($G=4$ orange,
  $G=16$ purple). The $x$-axis is training progress normalized per
  $G$ so the two curves overlay. Solid lines are seed-averaged means
  over seeds $\{42, 3407\}$ with shaded $\pm$ s.d.\ bands.
  \textbf{Row 1, panels (a)--(c) -- gradient-norm trajectory.}
  Seed-averaged \texttt{train/grad\_norm} on a logarithmic $y$-axis.
  Proposition~\ref{prop:stab} predicts that the linearized dynamics
  \eqref{eq:main} become unstable iff
  $K\,\eta\,K_\mathrm{ref}>1-\mu$, in which case the effective
  damping $C_\mathrm{eff}$ of \eqref{eq:final} becomes negative and
  the characteristic roots cross into the right half-plane. The
  empirical witness is monotone exponential growth (a straight
  ascending line on a log scale) or abrupt spikes in
  $\|\nabla J\|$; a bounded, weakly oscillating trace is the
  stable-regime signature.
  \textbf{Row 2, panels (d)--(f) -- importance-sampling drift.}
  Seed-averaged
  \texttt{train/sampling/importance\_sampling\_ratio/max}: the
  maximum across the batch of $\pi_\theta(o\mid q)/\pi_\mathrm{old}(o\mid q)$,
  where $\pi_\mathrm{old}=\pi_{\theta_{k-K_\mathrm{ref}}}$ is the
  refreshed sampling policy of \eqref{eq:update}. This is the direct
  witness of the off-policy lag analysed in Lemma~\ref{lem:lag}:
  as $\tau=K_\mathrm{ref}\delta$ grows, $\pi_\mathrm{old}$ and
  $\pi_\theta$ separate, the IS-ratio support widens, and the
  $-K\tau$ damping reduction in $C_\mathrm{eff}=c_\mu-K\tau$ becomes
  material. The IS-ratio max blowing up is the practical sign that
  the system is approaching the threshold \eqref{eq:threshold}.
  \textbf{Reading.} A bounded \texttt{grad\_norm} in row~1 together
  with an IS-ratio max staying near unity in row~2 indicates the
  training is in the stable regime
  $K\,\eta\,K_\mathrm{ref}<1-\mu$, consistent with the overdamped
  saturation observed in Figure~\ref{fig:grpo_dynamics_gsm8k}. Divergence
  of either signal locates the training fraction at which the
  linearized model of Theorem~\ref{thm:main} becomes inapplicable
  because the system has crossed the stability threshold of
  Proposition~\ref{prop:stab}.%
  }
  \label{fig:gsm8k-stability}
\end{figure}

\subsection{Exercising the Oscillatory Regime in an Exact Reduction}
\label{sec:oscillatory}

The experiments of the preceding subsections all lie deep in the overdamped
regime that Proposition~\ref{prop:zeta} identifies with the small-$\eta$
configuration of Table~\ref{tab:gsm8k_config}, and so they leave untested
precisely the two predictions that separate a second-order model from the
first-order saturation laws it generalizes: the stability threshold
$K\eta K_{\mathrm{ref}}>1-\mu$ of Proposition~\ref{prop:stab} and the
overdamped-to-oscillatory transition of Section~\ref{sec:saturation}. We exercise
both here. To isolate the predicted physics from the deep-network approximation,
we run the reduction in the setting where Assumption~\ref{ass:mf} holds
\emph{exactly}: a softmax policy over a finite completion set, for which the order
parameter $p(\theta)=\mathbb{E}_{o\sim\pi_\theta}[R]$ closes on itself and the
gradient-flow projection of \eqref{eq:update} is exact rather than approximate.
Any oscillation observed here is therefore attributable to the inertial term of
Theorem~\ref{thm:main} and not to a breakdown of the mean-field premise.

\paragraph{Setup.} We take a single prompt with $|\mathcal{O}|=8$ completions
carrying fixed verifiable rewards, a uniform reference policy $\pi_{\mathrm{ref}}$,
and initialization $\theta_0$ at the uniform policy, so the trajectory starts
from rest ($\dot p(0)=0$) at $R_0=p_{\mathrm{ref}}$. We evolve the exact
heavy-ball GRPO map \eqref{eq:update} with a genuine stale-policy refresh
$\theta_{\mathrm{old}}=\theta_{k-K_{\mathrm{ref}}}$, computing the group-relative
advantage \eqref{eq:adv} under $\pi_{\theta_{\mathrm{old}}}$ in closed form (the
$G\!\to\!\infty$ deterministic limit, in which the importance ratio of
\eqref{eq:obj} cancels the sampling measure); the $1/G$ forcing of
Lemma~\ref{lem:noise} is suppressed so that the deterministic dynamics are seen
without stationary fluctuation. Crucially, the stiffness $K$ is \emph{measured}
once, from an independent stable baseline ($\mu=0$, $K_{\mathrm{ref}}=1$, small
$\eta$), by reading the early-training relaxation rate $\lambda$ off
$-\log\!\big(p^\star-p(t)\big)$ and inverting the overdamped slow-pole relation of
Corollary~\ref{cor:sat}; the recovered $p^\star=1.35$ and $K=0.12$ are then held
fixed and used only to \emph{predict} the critical refresh interval
$K_{\mathrm{ref}}^\star=(1-\mu)/(K\eta)$. The analytic curves overlaid below carry
no per-trajectory fit.

\paragraph{The transition appears as predicted (Figure~\ref{fig:sweep}).}
Fixing $(\eta,\mu,\beta)=(0.30,0.80,1.00)$ and sweeping $K_{\mathrm{ref}}$, the
exact order parameter traces the full cascade Theorem~\ref{thm:main} prescribes.
At small $K_{\mathrm{ref}}$ the response is monotone, a single overshoot-free
approach to $p^\star$; as $K_{\mathrm{ref}}$ grows the off-policy lag of
Lemma~\ref{lem:lag} erodes the effective damping
$C_{\mathrm{eff}}=c_\mu-K\tau$, the damping ratio \eqref{eq:zeta} drops through
unity, and the trajectory develops first an overshoot and then a sequence of
damped oscillations about $p^\star$. The inertia responsible for the ringing is
the same momentum-plus-lag mass $M_{\mathrm{eff}}$ of Lemma~\ref{lem:lag} that
supplies the slow-start inflection of Proposition~\ref{prop:threephase}: the two
phenomena are the overdamped and underdamped faces of one second-order system,
and a first-order saturation law \eqref{eq:emp} can represent neither.

\paragraph{The boundary tracks the prediction where the linearization holds
(Figure~\ref{fig:phase}).} Classifying the exact dynamics over an
$(\eta,K_{\mathrm{ref}})$ grid into monotone, overshooting, and oscillatory
regimes, the empirical monotone--oscillatory frontier follows the hyperbolic
locus $K\eta K_{\mathrm{ref}}=\text{const}$ that Proposition~\ref{prop:stab}
predicts. In the small-$\eta$, large-$K_{\mathrm{ref}}$ corner---where the local
linearization of Remark~\ref{rem:scope} is most trustworthy---the measured
frontier sits essentially on the predicted threshold line
$K\eta K_{\mathrm{ref}}=1-\mu$, drawn with the independently measured $K$ and no
free parameter. This is the central positive result: the refresh-interval
threshold of Proposition~\ref{prop:stab} is not merely an artifact of the
derivation but a property of the exact dynamics, pinned in absolute terms by a
slope measured elsewhere.

\paragraph{Two qualifications, both informative.} First, the $\zeta=1$ onset
curve obtained from the $O(\tau^2)$-truncated coefficients \eqref{eq:final} sits
\emph{below} the empirical onset: the truncated model calls for oscillation at
smaller $\eta$ than the exact system actually exhibits, so the second-order
expansion is conservative about where ringing begins even as it captures the
threshold's functional form. The gap narrows toward large $K_{\mathrm{ref}}$ and
widens at large $\eta$, exactly where a low-order expansion of the delay
$e(t-\tau)$ should be expected to degrade. Second, the exact trajectories never
diverge: the linear instability of Proposition~\ref{prop:stab}, in which the
characteristic roots cross into the right half-plane, manifests in the nonlinear
system as \emph{persistent bounded ringing} rather than as $p\to\infty$, because
the order parameter is confined to $[\min_o R,\max_o R]$ and the softmax
saturates. The unbounded blow-up is therefore a property of the linearization
alone; Proposition~\ref{prop:stab} correctly locates the loss of monotone
stability but overstates its consequence, the oscillation amplitude being
regularized by the same nonlinearity that Assumption~\ref{ass:mf} linearizes
away.

\paragraph{Scope.} This study does for the dynamical predictions what the
mode-collapse and policy-concentration diagnostics do for the mean-field premise:
it identifies the regime in which the closed form can be trusted and the manner
in which it first fails. The overdamped-to-oscillatory transition and the
Proposition~\ref{prop:stab} threshold are confirmed in the setting where
Assumption~\ref{ass:mf} is exact, locating the truncation-induced error in the
onset rather than in the threshold and showing the predicted divergence to be a
nonlinearly-regularized ringing. What remains is the deep-network demonstration:
repeating the $K_{\mathrm{ref}}$ sweep (or, equivalently, raising $\eta$) on the
models of Section~\ref{subsec:gpro_training_dynamics} and reading the oscillation
off the reward curve while the importance-sampling drift and gradient norm of
Figure~\ref{fig:gsm8k-stability} witness the crossing. There
Assumption~\ref{ass:mf} is only approximate, so the exact-reduction boundary
established here is the appropriate reference against which to measure the
deviation.

\begin{figure}[!htbp]
  \centering
  \includegraphics[width=\linewidth]{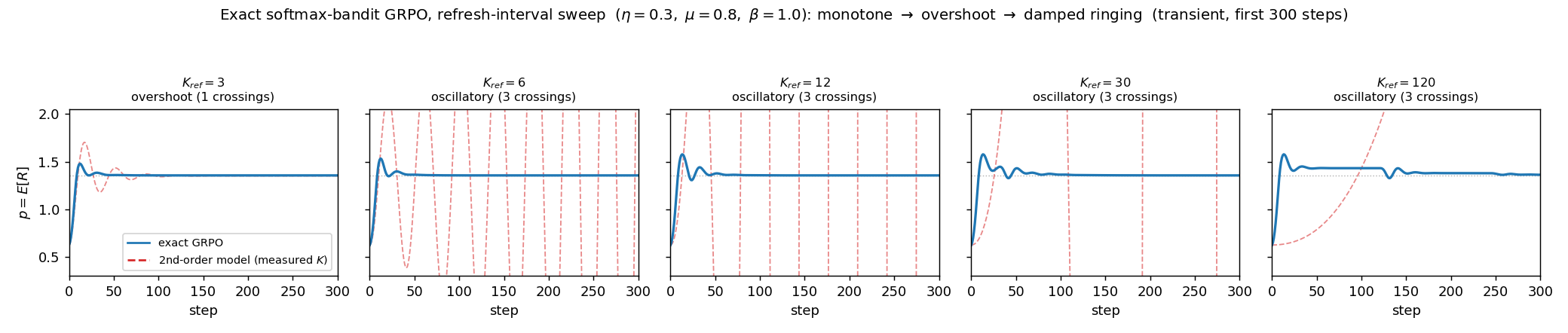}
  \caption{\textbf{Overdamped-to-oscillatory transition in the exact softmax
  reduction.} Exact GRPO order parameter $p(t)=\mathbb{E}_{\pi_\theta}[R]$
  (solid) under the heavy-ball map \eqref{eq:update} with a stale-policy refresh
  $\theta_{\mathrm{old}}=\theta_{k-K_{\mathrm{ref}}}$, at fixed
  $(\eta,\mu,\beta)=(0.30,0.80,1.00)$ and increasing refresh interval
  $K_{\mathrm{ref}}$ (transient, first $300$ steps; dotted line marks
  $p^\star$). As $K_{\mathrm{ref}}$ grows the off-policy lag of
  Lemma~\ref{lem:lag} erodes the effective damping \eqref{eq:zeta} and the
  response passes from a monotone approach, through a single overshoot, to a
  sequence of damped oscillations about $p^\star$. Dashed curves are the
  second-order prediction of Theorem~\ref{thm:main} evaluated with the
  \emph{independently measured} stiffness $K$ (no per-trajectory fit); in this
  regime the $O(\tau^2)$-truncated coefficients place the system past $\zeta=0$,
  over-stating the amplitude that the bounded nonlinear dynamics actually
  realize.}
  \label{fig:sweep}
\end{figure}

\begin{figure}[!htbp]
  \centering
  \includegraphics[width=0.72\linewidth]{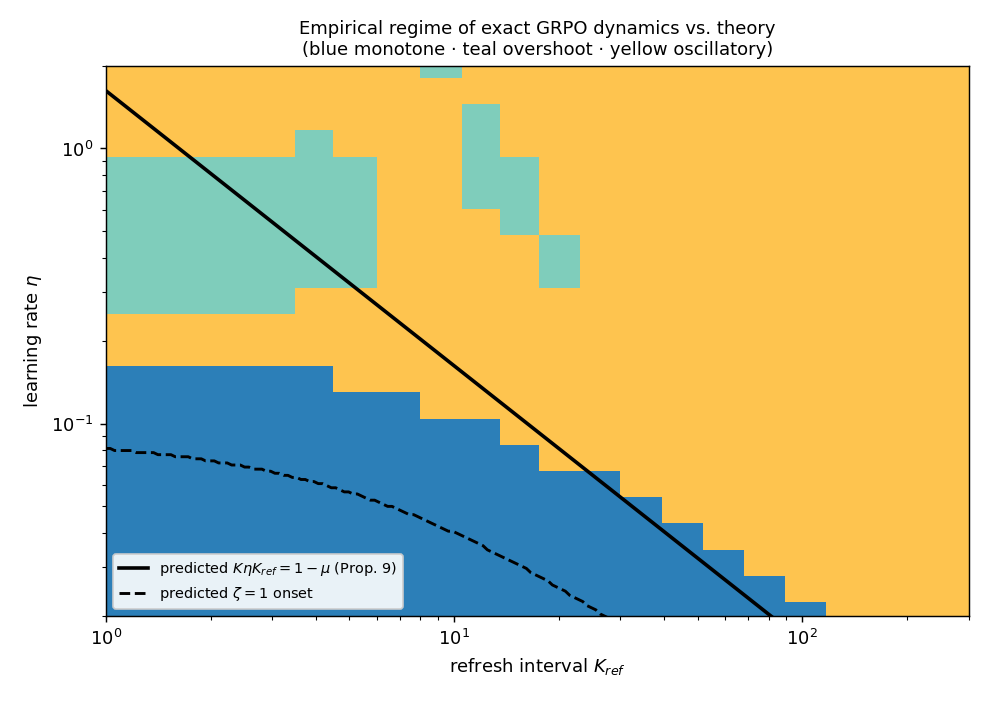}
  \caption{\textbf{Empirical regime of the exact dynamics versus the
  Proposition~\ref{prop:stab} threshold.} Each cell classifies the exact
  softmax-bandit trajectory at $(\eta,K_{\mathrm{ref}})$, $\mu=0.80$,
  $\beta=1.00$, as monotone (blue), overshooting (teal), or oscillatory
  (yellow). The solid line is the predicted stability threshold
  $K\eta K_{\mathrm{ref}}=1-\mu$ of Proposition~\ref{prop:stab}, drawn with the
  independently measured $K$; the dashed line is the $\zeta=1$ onset implied by
  the $O(\tau^2)$ coefficients \eqref{eq:final}. In the small-$\eta$,
  large-$K_{\mathrm{ref}}$ corner, where the linearization of
  Remark~\ref{rem:scope} is most accurate, the empirical monotone--oscillatory
  frontier coincides with the predicted threshold. The truncated $\zeta=1$ onset
  lies below the empirical onset, quantifying the conservatism of the
  second-order delay expansion.}
  \label{fig:phase}
\end{figure}


\subsection{Out-of-Distribution Benchmarking}

Table~\ref{tab:gsm8k_benchmark} reports out-of-distribution transfer: all three
models are trained with GRPO on GSM8K and evaluated, without further tuning, on
eight math benchmarks spanning grade-school word problems (GSM-Plus, MetaMathQA)
through competition mathematics (AIME-2026, HMMT-2025, IMO-Bench). Entries are
pass@$k$ means $\pm$ standard deviation over seeds, with green marking
improvements over the untrained base, bold-green the best configuration for each
model, and red the rare regressions. The headline result is that GRPO on a
single grade-school dataset transfers broadly and positively: the large majority
of cells are green for both group sizes and all three models, and pass@5 exceeds
pass@1 throughout, so the gains reflect genuine accuracy rather than a shift in
sampling temperature.

The size of the transfer is strongly graded by the distance between the training
and evaluation distributions. Gains are largest on benchmarks closest to GSM8K:
Nemotron-4B more than doubles its GSM-Plus pass@1 (9.11 to 20.32 at $G=16$), and
DeepSeek-7B improves GSM-Plus pass@1 by roughly fifteen points (25.64 to 40.25 at
$G=4$). They shrink markedly on the hardest competition sets, where the two
non-reasoning models (DeepSeek-7B, Nemotron-4B) remain at or near zero on
AIME-2026 and HMMT-2025 regardless of training, and only the reasoning-distilled
DeepSeek-R1-Qwen-1.5B shows movement there---AIME-2026 pass@1 rising from 1.33 to
6.17 and HMMT-2025 pass@5 from 10.00 to 16.67 at $G=16$. GRPO on in-distribution
data thus amplifies capabilities the base model already possesses far more than
it confers new ones on benchmarks well outside the training distribution.

The two group sizes are broadly comparable. $G=16$ secures most of the
per-model best results, but $G=4$ is competitive across the board and is the
stronger setting in several cases---most strikingly DeepSeek-7B on GSM-Plus
(40.25 versus 30.83 pass@1) and Qwen-1.5B on HMMT-2025 pass@1 (7.33 versus
6.33). This near-parity of the two budgets is consistent with the group-size
invariance of the deterministic dynamics (Proposition~\ref{prop:group}
and Remark~\ref{rem:corr}): quadrupling the rollout count leaves the
mean transfer performance largely unchanged, matching the prediction that small
groups track large groups up to noise. The few regressions are not uniformly
distributed but concentrate on NuminaMath, where several configurations fall
slightly below base (e.g.\ DeepSeek-7B pass@1 from 16.97 to 15.96 at $G=16$),
indicating that transfer is overwhelmingly but not universally beneficial and
can mildly degrade on particular target distributions.

\begin{table}[h]
\centering
\caption{Out-of-distribution evaluation of GRPO models trained on GSM8K dataset across math benchmarks using pass@k (\%). Results are reported as mean $\pm$ standard deviation over random seeds. For MetaMathQA and OpenMath2 datasets, we randomly sample 2000 examples for evaluation. For NuminaMath, we skip one proof problem that is not objectively verifiable.}
\label{tab:gsm8k_benchmark}
\scriptsize
\setlength{\tabcolsep}{3pt}
\resizebox{\textwidth}{!}{%
\begin{tabular}{llccccccccc}
\toprule
& &
\multicolumn{3}{c}{\textbf{DeepSeek-7B}}
& \multicolumn{3}{c}{\textbf{DeepSeek-R1-Qwen-1.5B}}
& \multicolumn{3}{c}{\textbf{Nemotron-4B}} \\
\cmidrule(lr){3-5}
\cmidrule(lr){6-8}
\cmidrule(lr){9-11}
\textbf{Benchmark} & \textbf{Metric}
& \textbf{Base} & \textbf{G=16} & \textbf{G=4}
& \textbf{Base} & \textbf{G=16} & \textbf{G=4}
& \textbf{Base} & \textbf{G=16} & \textbf{G=4} \\
\midrule

\multirow{2}{*}{GSM-Plus (2400)}
& pass@1
& 25.64 & {\color{ForestGreen}30.83{\scriptsize$\pm$0.71}} & {\color{ForestGreen}40.25{\scriptsize$\pm$0.08}}
& 33.48 & {\color{ForestGreen}36.52{\scriptsize$\pm$1.32}} & {\color{ForestGreen}34.80{\scriptsize$\pm$1.50}}
& 9.11 & {\color{ForestGreen}\textbf{20.32{\scriptsize$\pm$5.55}}} & {\color{ForestGreen}19.46{\scriptsize$\pm$2.92}} \\
& pass@5
& 49.25 & {\color{ForestGreen}52.19{\scriptsize$\pm$0.91}} & {\color{ForestGreen}60.50{\scriptsize$\pm$0.41}}
& 55.54 & {\color{ForestGreen}57.86{\scriptsize$\pm$0.67}} & {\color{ForestGreen}56.44{\scriptsize$\pm$1.15}}
& 21.92 & {\color{ForestGreen}40.86{\scriptsize$\pm$8.52}} & {\color{darkgreen}\textbf{42.02{\scriptsize$\pm$3.86}}} \\

\midrule
\multirow{2}{*}{Math-500 (500)}
& pass@1
& 15.00 & {\color{ForestGreen}16.12{\scriptsize$\pm$0.11}} & {\color{ForestGreen}15.48{\scriptsize$\pm$1.07}}
& 70.80 & {\color{ForestGreen}75.74{\scriptsize$\pm$0.59}} & {\color{ForestGreen}73.18{\scriptsize$\pm$0.08}}
& 11.64 & {\color{darkgreen}\textbf{13.34{\scriptsize$\pm$0.31}}} & {\color{ForestGreen}12.62{\scriptsize$\pm$0.93}} \\
& pass@5
& 32.40 & {\color{ForestGreen}33.50{\scriptsize$\pm$0.14}} & {\color{ForestGreen}33.30{\scriptsize$\pm$3.54}}
& 85.20 & {\color{darkgreen}\textbf{89.60{\scriptsize$\pm$0.85}}} & {\color{ForestGreen}88.70{\scriptsize$\pm$1.84}}
& 31.40 & {\color{ForestGreen}31.50{\scriptsize$\pm$0.42}} & {\color{ForestGreen}31.60{\scriptsize$\pm$1.98}} \\

\midrule
\multirow{2}{*}{AIME-2026 (30)}
& pass@1
& 0.00 & {0.00{\scriptsize$\pm$0.00}} & {\color{ForestGreen}0.34{\scriptsize$\pm$0.47}}
& 1.33 & {\color{darkgreen}\textbf{6.17{\scriptsize$\pm$0.23}}} & {\color{ForestGreen}4.50{\scriptsize$\pm$0.24}}
& 0.00 & {0.00{\scriptsize$\pm$0.00}} & {0.00{\scriptsize$\pm$0.00}} \\
& pass@5
& 0.00 & {0.00{\scriptsize$\pm$0.00}} & {\color{ForestGreen}1.67{\scriptsize$\pm$2.35}}
& 5.19 & {\color{darkgreen}\textbf{11.78{\scriptsize$\pm$0.66}}} & {\color{ForestGreen}9.82{\scriptsize$\pm$0.14}}
& 0.00 & {0.00{\scriptsize$\pm$0.00}} & {0.00{\scriptsize$\pm$0.00}} \\

\midrule
\multirow{2}{*}{MetaMathQA (2000)}
& pass@1
& 66.93 & {\color{ForestGreen}73.07{\scriptsize$\pm$0.42}} & {\color{ForestGreen}70.30{\scriptsize$\pm$0.91}}
& 82.71 & {\color{ForestGreen}86.93{\scriptsize$\pm$1.10}} & {\color{ForestGreen}85.01{\scriptsize$\pm$0.79}}
& 46.54 & {\color{darkgreen}\textbf{52.55{\scriptsize$\pm$0.01}}} & {\color{ForestGreen}49.73{\scriptsize$\pm$0.38}} \\
& pass@5
& 90.65 & {\color{ForestGreen}91.78{\scriptsize$\pm$0.46}} & {\color{ForestGreen}90.88{\scriptsize$\pm$0.81}}
& 95.95 & {\color{ForestGreen}96.53{\scriptsize$\pm$0.04}} & {\color{ForestGreen}96.45{\scriptsize$\pm$0.21}}
& 74.05 & {\color{darkgreen}\textbf{75.43{\scriptsize$\pm$0.04}}} & {\color{ForestGreen}75.15{\scriptsize$\pm$0.64}} \\

\midrule
\multirow{2}{*}{NuminaMath (99)}
& pass@1
& 16.97 & {\color{red}15.96{\scriptsize$\pm$0.86}} & {\color{ForestGreen}17.57{\scriptsize$\pm$1.14}}
& 42.42 & {\color{ForestGreen}45.96{\scriptsize$\pm$0.14}} & {\color{ForestGreen}44.95{\scriptsize$\pm$0.71}}
& 14.55 & {\color{darkgreen}\textbf{17.37{\scriptsize$\pm$2.86}}} & {\color{red}14.55{\scriptsize$\pm$2.28}} \\
& pass@5
& 30.30 & {\color{red}28.79{\scriptsize$\pm$0.71}} & {\color{darkgreen}\textbf{35.35{\scriptsize$\pm$0.00}}}
& 61.62 & {\color{red}60.61{\scriptsize$\pm$1.43}} & {\color{red}60.11{\scriptsize$\pm$0.71}}
& 33.33 & {\color{red}32.32{\scriptsize$\pm$2.86}} & {\color{red}30.81{\scriptsize$\pm$7.85}} \\

\midrule
\multirow{2}{*}{OpenMath2 (2000)}
& pass@1
& 26.88 & {\color{ForestGreen}29.26{\scriptsize$\pm$0.68}} & {\color{ForestGreen}27.84{\scriptsize$\pm$0.29}}
& 54.42 & {\color{ForestGreen}62.15{\scriptsize$\pm$0.09}} & {\color{ForestGreen}59.05{\scriptsize$\pm$1.12}}
& 20.91 & {\color{darkgreen}\textbf{24.52{\scriptsize$\pm$0.19}}} & {\color{ForestGreen}22.51{\scriptsize$\pm$0.59}} \\
& pass@5
& 49.15 & {\color{ForestGreen}49.65{\scriptsize$\pm$0.85}} & {\color{red}48.35{\scriptsize$\pm$0.78}}
& 75.95 & {\color{ForestGreen}79.58{\scriptsize$\pm$0.81}} & {\color{ForestGreen}77.88{\scriptsize$\pm$0.04}}
& 42.00 & {\color{darkgreen}\textbf{44.35{\scriptsize$\pm$0.28}}} & {\color{ForestGreen}42.73{\scriptsize$\pm$0.39}} \\

\midrule
\multirow{2}{*}{IMO-Bench (400)}
& pass@1
& 2.00 & {\color{ForestGreen}2.23{\scriptsize$\pm$0.32}} & {\color{ForestGreen}2.20{\scriptsize$\pm$0.07}}
& 3.75 & {\color{darkgreen}\textbf{4.88{\scriptsize$\pm$0.81}}} & {\color{ForestGreen}4.67{\scriptsize$\pm$0.46}}
& 2.35 & {\color{red}2.17{\scriptsize$\pm$0.11}} & {\color{darkgreen}\textbf{2.60{\scriptsize$\pm$0.64}}} \\
& pass@5
& 7.50 & {\color{ForestGreen}7.88{\scriptsize$\pm$1.24}} & {\color{darkgreen}\textbf{8.25{\scriptsize$\pm$0.71}}}
& 10.75 & {\color{darkgreen}\textbf{14.38{\scriptsize$\pm$2.65}}} & {\color{ForestGreen}13.38{\scriptsize$\pm$1.24}}
& 9.00 & {9.00{\scriptsize$\pm$0.35}} & {\color{red}8.88{\scriptsize$\pm$1.24}} \\

\midrule
\multirow{2}{*}{HMMT-2025 (30)}
& pass@1
& 0.00 & {\color{ForestGreen}0.34{\scriptsize$\pm$0.47}} & {0.00{\scriptsize$\pm$0.00}}
& 6.00 & {\color{ForestGreen}6.33{\scriptsize$\pm$0.47}} & {\color{darkgreen}\textbf{7.33{\scriptsize$\pm$0.94}}}
& 0.00 & {0.00{\scriptsize$\pm$0.00}} & {0.00{\scriptsize$\pm$0.00}} \\
& pass@5
& 0.00 & {\color{darkgreen}\textbf{1.67{\scriptsize$\pm$2.35}}} & {0.00{\scriptsize$\pm$0.00}}
& 10.00 & {\color{darkgreen}\textbf{16.67{\scriptsize$\pm$0.00}}} & {\color{ForestGreen}15.00{\scriptsize$\pm$2.36}}
& 0.00 & {0.00{\scriptsize$\pm$0.00}} & {0.00{\scriptsize$\pm$0.00}} \\

\bottomrule
\end{tabular}%
}
\end{table}

\section{Conclusion and Future Work}
\label{sec:conclusion}

We have given a reduced-order model of GRPO training that recasts the fitting
of reward curves as a mechanistic account of their shape, conditional on a
mean-field reduction. Under a single mean-field assumption
(Assumption~\ref{ass:mf}), and a controlled second-order truncation in the step
size and refresh lag, the GRPO update reduces to a stochastically-forced damped
oscillator (Theorem~\ref{thm:main}) whose mass, damping, and stiffness are fixed
in closed form by the optimizer hyperparameters together with a single measured
curvature scale. The reduction is not
merely descriptive: it subsumes the empirical single-exponential saturation law
as its overdamped limit (Corollary~\ref{cor:sat}), reinterprets the fitted plateau,
timescale, and size exponent as the fixed point, inverse stiffness, and
curvature-scaling exponent of the underlying potential (Equation~\eqref{eq:curv}),
and supplies, through the retained inertial term, the slow-start inflection that
the single exponential cannot represent (Proposition~\ref{prop:threephase}).
Empirically, the three-phase law fits the training-reward trajectories of three
models at two group sizes to $R^2\ge0.91$, and the predicted group-size invariance
of the deterministic dynamics (Proposition~\ref{prop:group}) is borne out on both
the mean trajectory and out-of-distribution transfer to eight benchmarks. Beyond
the closed form, the dynamical reading yields a set of diagnostics that separate
failure modes a reward curve alone conflates---train--eval decoupling, advantage
degeneracy, policy concentration, and dynamical instability---and thereby
localize, within a single run, where the closed-form description can and cannot
be trusted.

The principal limitation is now one of parameterization rather than of regime.
The two predictions that most sharply distinguish a second-order model from the
first-order saturation laws it generalizes---the stability threshold
$K\eta K_{\mathrm{ref}}>1-\mu$ of Proposition~\ref{prop:stab} and the
overdamped-to-oscillatory transition under refresh-interval sweeps
(Section~\ref{sec:saturation})---are not reachable by our GSM8K runs, because the
small learning rate of our configuration places every one deep in the overdamped
regime (Proposition~\ref{prop:zeta}). We therefore exercise them in the setting
where Assumption~\ref{ass:mf} holds exactly, a softmax-bandit reduction
(Section~\ref{sec:oscillatory}): sweeping $K_{\mathrm{ref}}$ drives the exact
dynamics across the predicted boundary, reproducing the monotone-to-oscillatory
cascade and placing the empirical frontier on the threshold line $K\eta
K_{\mathrm{ref}}=1-\mu$ drawn with an independently measured $K$ and no free
parameter. That study also sharpens the theory's reach: the $O(\tau^2)$ delay
truncation is conservative about the $\zeta=1$ onset, and the linear divergence of
Proposition~\ref{prop:stab} is realized in the bounded nonlinear system as
persistent ringing rather than blow-up. What remains is the deep-network
demonstration---repeating the sweep, or correspondingly raising $\eta$, on the
models of Section~\ref{subsec:gpro_training_dynamics}, where
Assumption~\ref{ass:mf} is only approximate and the exact-reduction boundary
established here is the reference against which to measure the deviation.

Three further threads are opened by the empirical analysis. First, our
diagnostics record a residual group-size dependence of the within-group spread
$\sigma_{\mathrm{adv}}$ that the temperature interpretation of
Lemma~\ref{lem:noise} does not predict; the
deterministic trajectory remains $G$-invariant, but $\sigma_{\mathrm{adv}}$
itself carries structure beyond the $1/G$ forcing scaling, and a model of that
structure would extend the account past the leading-order temperature picture.
Second, the reduction is local: its coefficients are linearizations at the fixed
point $p^\star$ (Remark~\ref{rem:scope}), so it describes the approach to
saturation rather than the global trajectory, and a treatment of the pre-linear
phase---where mode concentration and advantage degeneracy first appear---would
close the gap between the dynamics and the failure modes the diagnostics detect.
Third, the curvature-scaling exponent $K(M)\propto\eta^{-1}M^{-0.3}$ is inherited
from the empirical fits rather than derived; relating it to the
lazy/NTK-regime slowdown it resembles would convert the last fitted exponent in
the chain into a predicted one. Taken together, these directions point toward a
GRPO training theory that is predictive not only at the overdamped fixed point we
characterize here but across the regimes where current practice actually
operates.

\acks{The authors received no specific funding for this work. The authors
declare no competing interests.}

\newpage

\clearpage
\appendix

\section{Symbolic verification of the coefficients}
\label{app:symbolic}
 
The continuous-time coefficients of Lemma~\ref{lem:momentum} and the delay
corrections of Lemma~\ref{lem:lag} were verified by symbolic computation. A
second-order Taylor expansion of the heavy-ball map reproduces \eqref{eq:hb}, and
substituting the delayed restoring force into the resulting equation reproduces
\eqref{eq:effcoef} together with the sign condition \eqref{eq:threshold}.

\section{GSM8K Hyperparameters}
\label{app:gsm8k_hyperparams}

\begin{table}[H]
\centering
\caption{GSM8K training configuration for GRPO experiments.}
\label{tab:gsm8k_config}
\begin{tabular}{l l}
\toprule
\textbf{Category} & \textbf{Configuration} \\
\midrule

\multicolumn{2}{l}{\textit{Training}} \\
Learning Rate & $1 \times 10^{-5}$ \\
Epochs & 2 \\
Batch Size & 4 \\
Gradient Accumulation & 4 \\
Temperature & 0.8 \\
Max Prompt Length & 256 \\
Max Completion Length & 786 \\
Logging Steps & 1 \\
Eval Steps & 100 \\
KL Coefficient ($\beta$) & 0.005 \\
Clipping Coefficient ($\epsilon$) & 0.2 \\
\midrule

\multicolumn{2}{l}{\textit{LoRA}} \\
Use LoRA & True \\
Rank ($r$) & 16 \\
Alpha & 64 \\
Dropout & 0.05 \\
Target Modules & q, k, v, o, up, down, gate projections \\
Merge Adapters & True \\

\bottomrule
\end{tabular}
\end{table}

\vskip 0.2in
\bibliography{sample}

\end{document}